\documentclass[11pt,a4paper]{article}
\usepackage[T1]{fontenc}
\usepackage{lmodern}
\usepackage[utf8]{inputenc}
\usepackage[margin=25mm]{geometry}
\usepackage{amsmath,amssymb,amsthm,mathrsfs}
\usepackage{graphicx,booktabs,array}
\usepackage[numbers,sort&compress]{natbib}
\usepackage{algorithm,algorithmic}
\usepackage{xcolor}
\usepackage{xurl}
\usepackage[hidelinks]{hyperref}
\hypersetup{
  pdftitle={CovFedAvg: Local Low-Loss Curves and Federated Model Selection},
  pdfauthor={Federico Lucchetti and Marcus Völp},
  pdfsubject={Federated learning under feature heterogeneity},
  pdfkeywords={federated learning, non-IID data, model aggregation, mode connectivity}
}
\usepackage{microtype}
\usepackage{placeins}
\newtheorem{theorem}{Theorem}[section]
\newtheorem{proposition}[theorem]{Proposition}

\newtheorem{corollary}[theorem]{Corollary}
\theoremstyle{definition}

\newtheorem{assumption}[theorem]{Assumption}
\newtheorem{remark}[theorem]{Remark}
\newcommand{\R}{\mathbb{R}}
\newcommand{\E}{\mathbb{E}}
\newcommand{\covfed}{\textsc{CovFedAvg}}
\newcommand{\covfedcal}{\covfed{} (calibration)}
\newcommand{\covfedadapt}{\covfed{} (calibration + adapters)}
\DeclareMathOperator*{\argmin}{arg\,min}
\graphicspath{{./}}
\title{CovFedAvg: Local Low-Loss Curves\\
and Federated Model Selection}
\author{
Federico Lucchetti\textsuperscript{1}\quad Marcus Völp\textsuperscript{2}\\[0.5em]
\small\textsuperscript{1}Center for Psychiatric Neuroscience, Department of Psychiatry,\\
\small Lausanne University Hospital and the University of Lausanne,\\
\small Lausanne, Switzerland\\[0.3em]
\small\textsuperscript{2}Interdisciplinary Centre for Security, Reliability and Trust (SnT),\\
\small University of Luxembourg, Luxembourg
}
\date{29 October 2026}
\begin{document}
\maketitle
\begin{abstract}
Clients observing different parts of a prediction task may learn locally
equivalent models with incompatible behavior outside their observed support.
We study whether redundant models permit local low-loss paths toward a
compatible shared model. CovFedAvg retains local models and curve controls
on clients, fits a common endpoint through aggregated gradient contributions,
and selects an actual candidate model using distributed training losses.
No centralized calibration examples or client-to-client sample exchange are
required. Exact selection guarantees no higher evaluated training loss than
the previous model and the current-round FedAvg candidate; this is not a
global-optimality, stationarity, generalization, or differential-privacy
guarantee. We state conditional client-compatibility and approximate-selection
bounds and retain corrected transformation identities and a separate
full-gradient safeguard theorem for reference algorithms. A complementary
input-support example recovers the shared absolute-value regression, while
matched direct fine-tuning also solves it. Across a synthetic support task,
recorded activity-recognition subjects, and participating and future
gas-sensor batches, we evaluate five methods, five seeds, two network widths,
and 20 rounds without calibration data. At width 16, CovFedAvg raises
participating-Gas macro balanced accuracy from FedAvg's $88.88\%$ to $89.84\%$;
matched distributed fine-tuning reaches $91.65\%$. Results support feasibility
and modest gains over FedAvg, but do not establish a curve-specific advantage
over equally costly optimization. Exact-sum numerical experiments establish
data flow; end-to-end privacy requires additional mechanisms and analysis.
\end{abstract}

\section{Introduction}
Federated learning optimizes a shared predictor from data held by multiple
clients. FedAvg alternates local training with averaging of client parameters
\citep{mcmahan2017fedavg}. Its behavior depends on the number of local steps,
the learning rate, the model, and differences between client objectives.
Feature heterogeneity is one source of these differences: scanners, cameras,
measurement units, or acquisition conditions may change the input
distribution without changing the semantic meaning of a label. Such
heterogeneity is distinct from a change in class proportions and from an
arbitrary change in the labeling rule.

The motivating idea of covariant federated learning is to describe client
domains as transformations of a reference distribution and to account for
those transformations when combining models. The useful part of this
analogy is precise: transforming both the inputs and a compatible predictor
can preserve prediction risk. However, this does not imply that an unchanged
predictor has invariant risk, that every pair of client distributions admits
an invertible transformation, or that a low-loss curve connecting trained
models must exist. The present revision states these assumptions explicitly.

The mechanism of interest is more general than coordinate changes. When
clients observe complementary parts of the same input--output relationship,
some model directions can be unidentifiable locally. A redundant model may
incorporate another client's information through those directions while
preserving its original predictions. This motivates seeking compatible
representatives of local low-loss regions. A common labeling vocabulary or
model redundancy alone does not establish that those regions intersect.
Pure parameter symmetries can also preserve the same function everywhere
without adding any knowledge about an unobserved domain.

Low-loss paths motivate exploring parameter configurations
\citep{garipov2018mode}. The present method, CovFedAvg, keeps each client's
local model and path control on that client and fits a common endpoint using
local supervised losses and aggregate messages. It selects the model itself
rather than an arithmetic centroid of curve weights. This construction needs
neither server-held labeled calibration data nor sample sharing between
clients. Its simulation uses an aggregate-only interface; keeping examples
local is not a formal privacy guarantee. A free shared endpoint also means
that fitting is distributed optimization with geometric penalties, whose
extra gradient costs must be compared against direct optimization.

The main contribution is an executable local-curve fitting and model-selection
protocol, with conditional compatibility and exact or bounded-error selection
statements, evaluated against matched distributed fine-tuning and centroid
controls. We retain three reference components from the earlier revision:
correct joint pushforwards with explicit predictor compatibility; normalized
public-calibration centroid proposals and their limitations; and a separate
full-client gradient safeguard with a stationarity theorem. The latter
theorem does not automatically apply to CovFedAvg's selector. Earlier
calibration-based experiments remain labeled reference evidence rather than
results for the new method.

Statistical
heterogeneity alone does not imply that FedAvg must perform poorly: its
one-full-gradient-step average equals the centralized gradient step for any
client distributions. Additional local steps, gradient disagreement, model
capacity, and the evaluation domain determine how much heterogeneity matters.
The evaluation preserves favorable FedAvg results as well as failures.
Closely related
geometric aggregation already appears in FedGMCC
\citep{lucchetti2022fedgmcc}, and intersections of client mode connections
are studied by \citet{thakur2025clinical}. We do not claim invention of
geometric model aggregation or universal recovery under feature heterogeneity.

\section{Related work}
\label{sec:related}
FedAvg alternates local optimization and sample-weighted parameter
averaging~\citep{mcmahan2017fedavg}. FedProx adds a proximal penalty to the
local objective to control drift under statistical and systems
heterogeneity~\citep{li2020fedprox}. Both are relevant baselines; neither
method implies that arbitrary feature transformations preserve the same
optimal parameter vector. SCAFFOLD uses persistent client and server control
variates to correct client drift~\citep{karimireddy2020scaffold}. We include
a full-participation, sample-weighted Option II implementation and disclose
its unequal client minibatch-step counts; the published rates are not
silently applied to that extension.

Feature heterogeneity has established dedicated methods. FedBN retains
client-specific batch-normalization parameters and statistics while
aggregating the other model parameters~\citep{li2021fedbn}. This is a
change to parameter sharing throughout training, rather than a single
normalization of the input. FRAug trains a shared representation
generator and client-specific representation transformation networks to
augment local embeddings~\citep{chen2023fraug}. The reference centroid
construction uses public labeled calibration examples; this additional
resource must be disclosed in comparisons. The current CovFedAvg method
instead fits local curves through aggregate messages without central
calibration data, as evaluated in Section~\ref{sec:local-curve-experiments}.
The reference natural-domain evaluation
compares FedAvg, FedProx, SCAFFOLD, a centralized reference, fixed-adapter
ablations, and safeguards with no calibration, direct calibration, and curve
fitting. The two calibration proposal methods receive the same labeled
calibration examples and gradient-evaluation budget. The full-client risk
and gradient queries needed by the safeguards are separately reported.
Numerical comparisons with FedBN and FRAug remain open; fixed feature
standardization is not presented as either method.

The practical path construction is motivated by mode
connectivity~\citep{garipov2018mode}, which finds simple low-loss paths
between particular trained networks. Stochastic weight averaging provides
an empirical precedent for averaging nearby trained
weights~\citep{izmailov2018swa}. These observations do not guarantee that
the arithmetic centroid of an arbitrary low-loss curve also has low
loss. Parameter coordinates also matter: Git Re-Basin studies permutation
alignment before merging networks~\citep{ainsworth2023gitrebasin}. A
connection between model weights therefore cannot, by itself, identify a
feature-space transformation or quantify distributional heterogeneity.

Geometric aggregation has already been studied in federated learning.
FedGMCC combines Monte Carlo output probing, path connectivity, and
aggregation of connected client models within
clusters~\citep{lucchetti2022fedgmcc}; it is a particularly close existing
framework. Later studies analyze mode connectivity under
heterogeneity~\citep{zhou2023mode} and averaging across global
checkpoints~\citep{zhou2024averaging}. FedGuCci uses anchor models to
improve group connectivity during client
training~\citep{li2025fedgucci}. We consequently make no claim to introduce
geometric federated learning or path aggregation. The scope here is a
corrected formulation of the covariance assumptions, a local-curve and
shared-endpoint protocol without server calibration data, conditional risk
and selection bounds, and reproducible evaluations with matched distributed
optimization controls. Earlier normalized centroid rules and their
calibration-based evaluations are retained as reference constructions.
The literature selection
was checked against primary sources on 28 October 2026; it is not an
exhaustive survey or a claim of state-of-the-art performance.

\section{Problem and transformation model}\label{sec:problem}
\subsection{Federated objective}
Let client $k\in\{1,\ldots,K\}$ hold a training set $D_k$ of size $n_k$,
drawn from a joint distribution $P_k$ on $\mathcal X\times\mathcal Y$.
Let $\alpha_k=n_k/\sum_j n_j$, unless a different target population is
explicitly chosen. A common parameter vector $w\in\R^p$ specifies a
predictor $f_w$. The population and empirical objectives are
\begin{equation}\label{eq:objective}
 F_k(w)=\E_{(X,Y)\sim P_k}\ell(f_w(X),Y),\qquad
 F(w)=\sum_{k=1}^K\alpha_k F_k(w),\qquad
 \widehat F_k(w)=\frac{1}{n_k}\sum_{(x,y)\in D_k}\ell(f_w(x),y).
\end{equation}
For classification, $f_w(x)$ is a probability vector and
$\ell(f_w(x),y)=-\log f_w(x)_y$. The global distribution is the
\emph{normalized mixture} $P=\sum_k\alpha_k P_k$.
Minimizing $F$ is a population tradeoff: a global optimizer need not achieve
the minimum of every $F_k$. Even two quadratic client objectives with
different minimizers generally cannot be minimized simultaneously.

\subsection{Label-preserving domains are an assumption}
Suppose a reference pair $(Z,Y)$ has distribution $P_0$, and client $k$
observes $(T_k(Z),Y)$. In measure notation,
\begin{equation}\label{eq:pushforward-model}
 P_k=(T_k,\mathrm{id}_{\mathcal Y})_\# P_0.
\end{equation}
The pushforward is defined for any measurable $T_k$, including noninvertible
maps. To use an inverse density formula we additionally require that $T_k$
be a $C^1$ diffeomorphism between open input domains with nonsingular
Jacobian. In this case, for each discrete label $y$,
\begin{equation}\label{eq:joint-density}
 p_k(x,y)=p_0(T_k^{-1}(x),y)\left|\det D T_k^{-1}(x)\right|.
\end{equation}
This is a formula for the joint input vector. A product of individual pixel
Jacobians is valid only for an appropriately separable transformation and
does not describe arbitrary coupled image transformations.

Invertibility is a modeling restriction, not a consequence of writing a
symbol $T$ or $\Gamma$. Positive coordinate scalings, translations, and
nonsingular affine maps satisfy it. Cropping, clipping, quantization, and
many style transformations do not. A continuous interpolation between two
invertible maps can also be singular: interpolation between $I$ and $-I$
passes through the zero map. The affine experimental transformations have
strictly positive scales and use no clipping, so they satisfy the restriction
on the continuous ambient input space. The observed benchmark datasets are
finite empirical measures; their transformation is understood as a
pushforward, without assuming that the empirical measures have densities.

Figure~\ref{fig:original-transforms} retains the original visual motivation:
client observations may differ by uniform intensity changes, spatially
varying transformations, or more complex deformations. Its $\Gamma$ notation
corresponds to the input map $T$ used here. For the label-preserving model,
such changes must retain the semantic class, for example different writing
styles of the same digit. A visual interpolation alone does not establish
that requirement or an invertible transformation between distributions.

\begin{figure}[tbp]
\centering
\includegraphics[width=.62\linewidth]{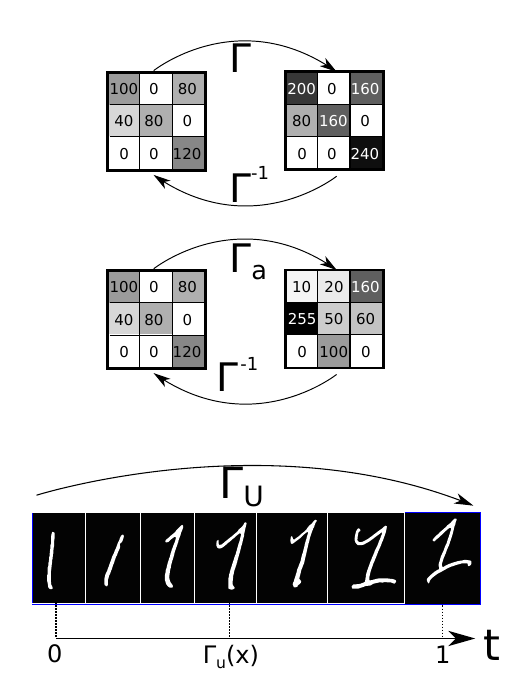}
\caption{Original schematic motivating transformed client domains. Top:
uniform intensity scaling; middle: spatially varying pixel changes;
bottom: a parameterized visual deformation. The inverse arrows require an
invertible map, which the finite examples do not establish. The digit
interpolation may change class; it illustrates deformation rather than the
label-preserving assumption of the formal model.}
\label{fig:original-transforms}
\end{figure}

Equation~\eqref{eq:pushforward-model} preserves $P_k(Y)=P_0(Y)$, but it does
not generally preserve the labeling conditional in the observed coordinates:
\begin{equation}
 P_k(Y\mid X=x)=P_0(Y\mid Z=T_k^{-1}(x)).
\end{equation}
Consequently, label-preserving feature transformations are not synonymous
with classical covariate shift, which requires a common $P(Y\mid X)$.
Nor is a change in $P(X\mid Y)$ alone a definition of concept shift. We use
\emph{transformed domains} for the present model and make no assertion that
it covers every form of statistical heterogeneity.

\subsection{What risk covariance does and does not mean}
If the model class admits a parameter action $\rho(T)$ satisfying
\begin{equation}\label{eq:model-action}
 f_{\rho(T)w}(Tz)=f_w(z),
\end{equation}
then transforming the data and predictor together preserves risk, as proved
in Section~\ref{sec:guarantees}. For example, if a network starts with an
unconstrained affine layer $Wz+b$ and $Tz=Az+c$ with $A$ nonsingular,
the transformed first layer can be chosen as
$W'=WA^{-1}$ and $b'=b-WA^{-1}c$, with later layers unchanged.
This example establishes an exact compatibility property for that model
class; it is not a claim that the server can identify $A$ from model weights.

A common transformation applied to identical client distributions leaves
those distributions identical. It does not imply invariant loss at fixed
weights. Conversely, different transformations can preserve the same
distribution, so a client-specific transformation need not create
heterogeneity. For example, rotations preserve an isotropic Gaussian.
The mean of per-coordinate distribution distances is insufficient to
identify general joint shifts: $(Z,Z)$ and $(Z,-Z)$ have identical
one-dimensional marginals for symmetric $Z$, but different dependence.
We do not infer a distance between client distributions from a curve in
parameter space.

The action~\eqref{eq:model-action}, when available, supplies client-specific
predictors with equal reference risk. A single predictor acting on raw
inputs from an unknown mixture of domains may not attain all those risks.
Compatible transformations, sufficient model capacity, and informative
observations remain substantive assumptions.

\FloatBarrier
\section{Reference construction: public-data curve centroids}\label{sec:algorithm}
We retain the curve proposal and the empirical guard used in the
controlled pilot. Section~\ref{sec:safeguarded} strengthens the acceptance
rule using actual client training objectives and fixed input adapters; that
variant is used in the reference natural-domain evaluation.
Section~\ref{sec:local-curves} introduces the current local-curve and model
selection construction, which eliminates the server calibration requirement
and is evaluated separately. These are different algorithms and access models.

\subsection{Geometric rationale}
Figure~\ref{fig:original-geometry} retains the original motivation for
exploring curves: distinct client solutions may occupy a curved low-loss
region, while their straight parameter average can cross a higher-loss
region. A learned path explores additional parameter configurations and
motivates a candidate aggregation rule. The picture is an idealized
possibility, not evidence that a suitable path exists for every client pair.
Input transformations and weight-space paths are separate objects; the
former do not establish the latter. The implemented construction below uses
paths from the FedAvg anchor to individual client models, followed by an
explicit evaluation of their normalized centroid. The counterexample in
Figure~\ref{fig:counterexample} explains why a low-loss path alone is
insufficient to certify that centroid.

\begin{figure}[tbp]
\centering
\includegraphics[width=.62\linewidth]{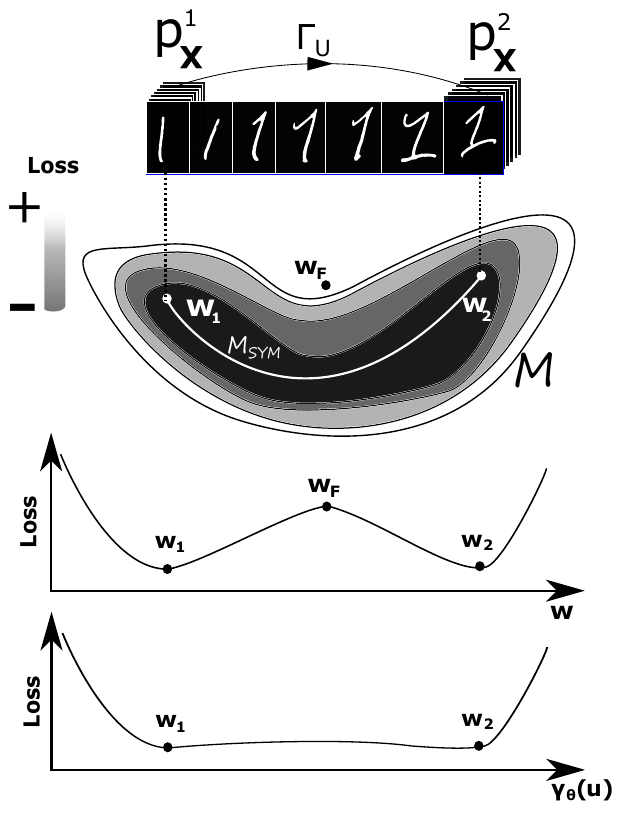}
\caption{Original geometric rationale for curve aggregation (idealized
schematic). Models $w_1,w_2$ can lie in a curved low-loss region even when
their Euclidean average $w_F$ has higher loss; the bottom profile motivates
seeking a low-loss path $\gamma_\theta(u)$. The original labels $\mathcal M$
and $\mathcal M_{\mathrm{SYM}}$ illustrate this picture and do not assert
an inferred manifold or connection. Neither low centroid risk nor a mapping
from input deformations to parameter paths follows from the sketch. The
digit strip is a deformation analogy and may change class.}
\label{fig:original-geometry}
\end{figure}

\subsection{Access model}
The original suggestion that random input probes reveal client training
losses does not provide an identifiable optimization problem: agreement
away from the data support need not imply agreement on the client task,
and an unlabeled probe does not measure supervised risk. Here we choose
an explicit, stronger access model. The server has a labeled public
calibration set $D_{\mathrm{cal}}$ representative of the target mixture,
and a disjoint public validation set $D_{\mathrm{val}}$. Both are separated
from local training and test data. Define their empirical risks
$F_{\mathrm{cal}}$ and $F_{\mathrm{val}}$ in the usual way.
These data are additional information, and their construction and sizes
are part of the experimental protocol. They are not recovered from client
weights. In a deployment they must come from an independently available,
appropriately collected source.

An alternative implementation could query weighted client losses and
gradients at proposed path points. This would optimize an empirical
federated objective without centralizing raw examples, but it adds
communication and exposes extra statistics. That alternative is not used
in the public-calibration reference experiments. The current construction
in Section~\ref{sec:local-curves} instead fits and selects using local
training losses and aggregate messages. Neither construction establishes differential
privacy or protection against model or gradient inversion.

\subsection{Continuous paths with a normalized centroid}
At communication round $t$, all clients start from $w_t$ and perform $E$
local epochs to return $v_{t,k}$. Write
\begin{equation}\label{eq:anchor}
 a_t=\sum_k\alpha_k v_{t,k}.
\end{equation}
For each client construct a polygonal path from the common anchor $a_t$
to $v_{t,k}$ through a trainable control $\theta_{t,k}$:
\begin{equation}\label{eq:polygon}
 \gamma_{t,k}(u)=
 \begin{cases}
 (1-2u)a_t+2u\theta_{t,k},&0\leq u\leq\tfrac12,\\
 2(1-u)\theta_{t,k}+(2u-1)v_{t,k},&\tfrac12\leq u\leq1.
 \end{cases}
\end{equation}
The two branches both equal $\theta_{t,k}$ at $u=1/2$. Unlike a differential
geometric connection, this is an ordinary continuous parameter-space path.
It can have a corner; smoothness at the control point is not required.
Initialize $\theta^0_{t,k}=(a_t+v_{t,k})/2$. Keeping the endpoints fixed,
approximately minimize
\begin{equation}\label{eq:curve-fitting}
 H_t(\theta)=\sum_k\alpha_k\left[
 \frac{1}{Q}\sum_{q=1}^Q F_{\mathrm{cal}}(\gamma_{t,k}(u_q))
 +\frac{\lambda}{2}\|\theta_{t,k}-\theta^0_{t,k}\|^2\right],
 \qquad (u_1,u_2,u_3)=(\tfrac14,\tfrac12,\tfrac34).
\end{equation}
The finite sample of path points is a numerical surrogate for an integral;
it does not certify the maximum loss along the whole path. The controls
are updated simultaneously by gradient descent on $H_t$, retaining the
client factor $\alpha_k$ in each control gradient.
The chain-rule factor multiplying $\nabla F_{\mathrm{cal}}$ is
$b(u)=2u$ on the first branch and $b(u)=2(1-u)$ on the second.
It is not valid simply to replace a parameter gradient with a control
gradient without this factor.

The exact average of each path and their weighted centroid are
\begin{equation}\label{eq:centroid}
 z_{t,k}=\int_0^1\gamma_{t,k}(u)\,du
 =\frac{a_t+v_{t,k}}{4}+\frac{\theta_{t,k}}{2},
 \qquad c_t=\sum_k\alpha_k z_{t,k}.
\end{equation}
The coefficients in each $z_{t,k}$ sum to one, and the client coefficients
also sum to one. With the midpoint initialization,
$z^0_{t,k}=(a_t+v_{t,k})/2$ and hence
$c^0_t=(a_t+\sum_k\alpha_kv_{t,k})/2=a_t$.
Thus disabling the fitting step recovers weighted FedAvg exactly for any
number of clients, any positive sample sizes, and any model dimension.
The identity also handles $K=1$; no factor $1/(K-1)$ is needed.

\subsection{Empirical guard and computation}
Accept the candidate only when its validation loss is no larger than the
anchor's loss:
\begin{equation}\label{eq:guard}
 w_{t+1}=
 \begin{cases}
 c_t,&F_{\mathrm{val}}(c_t)\leq F_{\mathrm{val}}(a_t),\\
 a_t,&\text{otherwise}.
 \end{cases}
\end{equation}
This guarantees an empirical comparison with this round's anchor. It does
not imply monotone validation loss across rounds, a population improvement,
or that the resulting trajectory dominates a separately trained FedAvg
trajectory. Reusing the validation set also introduces adaptive selection;
test data are never used in the guard.

\begin{algorithm}[t]
\caption{Revised \covfed\ with public calibration}
\label{alg:covfed}
\begin{algorithmic}[1]
\REQUIRE Initial $w_0$, weights $\alpha_k$, rounds $T$, fitting steps $J$,
calibration step size $\beta$, disjoint $D_{\mathrm{cal}},D_{\mathrm{val}}$.
\FOR{$t=0,\ldots,T-1$}
 \STATE All clients train from $w_t$; collect $v_{t,k}$.
 \STATE Compute $a_t=\sum_k\alpha_k v_{t,k}$.
 \STATE Initialize $\theta^0_{t,k}=(a_t+v_{t,k})/2$ for every $k$.
 \FOR{$j=0,\ldots,J-1$}
  \FOR{$k=1,\ldots,K$}
   \STATE Compute $\nabla_{\theta_k}H_t$ at $\theta^j$, including $\alpha_k$.
   \STATE Update $\theta^{j+1}_{t,k}$ using step size $\beta$.
  \ENDFOR
 \ENDFOR
 \STATE Compute $c_t$ by Equation~\eqref{eq:centroid}.
 \STATE Select $w_{t+1}$ by Equation~\eqref{eq:guard}.
\ENDFOR
\RETURN $w_T$.
\end{algorithmic}
\end{algorithm}

Each round requires $JKQ$ calibration forward/backward evaluations and two
validation forward evaluations in addition to ordinary local training.
For $K=4$, $J=3$, and $Q=3$, this is 36 calibration gradient evaluations.
Storage for the client models and controls scales as $O(Kp)$; sequential
control processing can reduce temporary path storage. In the public-data
implementation these operations do not add client communication, but they
add server computation and public data requirements. These costs are
different from fitting each path through federated gradient queries.

The matched calibration baseline performs $JKQ$ gradient updates to the
FedAvg anchor on the same public calibration set and uses the same
validation guard. A comparison with FedAvg alone would confound the effect
of the path representation with the effect of extra labeled data and
computation. Matching gradient-call counts controls that budget while
leaving the distinct update geometries and step-size sensitivity visible.

\FloatBarrier
\section{Mathematical guarantees and limits}
\label{sec:theory}
\label{sec:guarantees}

The results below separate three questions: when a feature transformation
preserves a prediction problem, how far local training can move from a
centralized trajectory, and when an aggregated update descends the global
risk. None of these properties implies the others without additional
assumptions. Throughout, $\alpha_k\geq 0$ and $\sum_{k=1}^K\alpha_k=1$.
We use the population risks
\[
 F_k(w)=\mathbb E_{(X,Y)\sim P_k}\ell(f_w(X),Y),
 \qquad F(w)=\sum_{k=1}^K\alpha_k F_k(w).
\]
The same statements apply to empirical risks when their assumptions hold.

\subsection{Feature transformations require a compatible model action}
\label{subsec:pushforward}

Let $\mathcal X,\mathcal X'$ be open subsets of $\mathbb R^d$, let
$\mathcal Y$ be a finite label set, and let
$T:\mathcal X\longrightarrow\mathcal X'$ be a $C^1$ diffeomorphism.
The label-preserving pushforward of a joint distribution $P$ is
\[
 P_T=(T,\operatorname{id}_{\mathcal Y})_\#P,
 \qquad (X_T,Y_T)=(T(X),Y).
\]
This defines a distribution regardless of whether a density exists.

\begin{proposition}[Pushforward and risk identity]
\label{prop:pushforward}
Suppose $P$ has a joint density $p(x,y)$ with respect to Lebesgue measure
in $x$ and counting measure in $y$. Then
\begin{equation}
\label{eq:joint-pushforward}
 p_T(x',y)=p(T^{-1}x',y)\left|\det D T^{-1}(x')\right|.
\end{equation}
In particular, $P_T(Y=y)=P(Y=y)$. Suppose, in addition, that the model
family admits a parameter transformation $\rho(T)$ satisfying the
explicit compatibility assumption
\begin{equation}
\label{eq:model-action-assumption}
 f_{\rho(T)w}(T x)=f_w(x)
 \quad\text{for $P$-almost every $(x,y)$ and every admissible $w$}.
\end{equation}
For any measurable loss with finite relevant expectations,
\begin{equation}
\label{eq:risk-covariance}
 F_{P_T}(\rho(T)w)=F_P(w),
 \qquad F_Q(v):=\mathbb E_Q\ell(f_v(X),Y).
\end{equation}
If $\rho(T)$ is a bijection between the two parameter spaces, it also
maps the complete set of minimizers of $F_P$ to that of $F_{P_T}$.
\end{proposition}

\begin{proof}
For each label $y$, apply the multivariate change-of-variables formula
to the joint measure on $\mathcal X\times\{y\}$. This gives
\eqref{eq:joint-pushforward}; integrating it over $x'$ preserves the
mass of each label. By the definition of a pushforward and
\eqref{eq:model-action-assumption},
\[
 F_{P_T}(\rho(T)w)
 =\mathbb E_P\ell(f_{\rho(T)w}(TX),Y)
 =\mathbb E_P\ell(f_w(X),Y)=F_P(w).
\]
If $\rho(T)$ is bijective, every candidate in the transformed parameter
space has a unique preimage. Comparing its risk with that of an arbitrary
transformed minimizer proves the minimizer correspondence in both
directions.
\end{proof}

For example, for $f_w(x)=w^\top x$ and $T(x)=Ax$ with $A$ invertible,
one can take $\rho(T)w=A^{-\top}w$. General nonlinear transformations need
not admit any compatible action within a fixed architecture. If a family
of transformations forms a group, calling $\rho$ a group action further
requires $\rho(S\circ T)=\rho(S)\circ\rho(T)$ and
$\rho(\operatorname{id})=\operatorname{id}$. A continuously parametrized
family alone does not establish these properties.

\begin{remark}[Limits of the transformation interpretation]
\label{rem:transformation-limits}
Equation~\eqref{eq:risk-covariance} compares transformed data with
transformed parameters; it does not assert invariance at fixed $w$.
Where conditional densities are defined,
$p_T(y\mid x')=p(y\mid T^{-1}x')$, which generally differs from
$p(y\mid x')$. Thus label preservation alone does not imply conventional
covariate shift, which requires a common conditional distribution in the
same input coordinates. Applying the same measurable transformation to
identical joint distributions preserves their equality, whether the
transformation is linear or nonlinear. The existence of a data
transformation does not establish a low-loss path between learned
parameters, recover the transformation from weights, or provide a common
optimizer for all clients.
\end{remark}

\subsection{A deterministic bound on local-training drift}
\label{subsec:drift}

We analyze one communication block, with $E$ full-gradient local steps
and a common initialization. Let $\eta>0$ and define
\begin{equation}
\label{eq:local-central-trajectories}
 \begin{split}
 u_{k,0}&=v_0=w,\\
 u_{k,s+1}&=u_{k,s}-\eta\nabla F_k(u_{k,s}),\\
 v_{s+1}&=v_s-\eta\nabla F(v_s),\qquad 0\leq s<E,\\
 a_E&=\sum_{k=1}^K\alpha_k u_{k,E}.
 \end{split}
\end{equation}
Here $v_s$ is a centralized trajectory, rather than a global minimizer.
These updates contain no stochastic gradient noise.

\begin{assumption}[Smoothness and gradient heterogeneity]
\label{ass:drift}
Each $\nabla F_k$ is $L_k$-Lipschitz on a convex region containing the
trajectories in~\eqref{eq:local-central-trajectories}, with $L_k\geq0$.
For constants $\delta_k\geq0$,
\[
 \|\nabla F_k(v_s)-\nabla F(v_s)\|\leq\delta_k,
 \qquad 0\leq s<E.
\]
\end{assumption}

\begin{proposition}[Local and aggregate drift]
\label{prop:local-drift}
Under Assumption~\ref{ass:drift}, write $q_k=1+\eta L_k$ and interpret
empty sums as zero. For every $0\leq s\leq E$,
\begin{equation}
\label{eq:local-drift-bound}
 \|u_{k,s}-v_s\|
 \leq \eta\delta_k\sum_{j=0}^{s-1}q_k^j.
\end{equation}
Consequently,
\begin{equation}
\label{eq:aggregate-drift-bound}
 \|a_E-v_E\|
 \leq \sum_k\alpha_k\eta\delta_k\sum_{j=0}^{E-1}q_k^j.
\end{equation}
The cancellation of the first-step client gradients also gives the
sharper bound
\begin{equation}
\label{eq:aggregate-drift-sharp}
 \|a_E-v_E\|
 \leq \eta^2\sum_k\alpha_k L_k\delta_k
       \sum_{s=0}^{E-1}\sum_{j=0}^{s-1}q_k^j.
\end{equation}
In particular, $a_1=v_1$. For any pair $k,l$,
\[
 \|u_{k,E}-u_{l,E}\|
 \leq\eta\delta_k\sum_{j=0}^{E-1}q_k^j
     +\eta\delta_l\sum_{j=0}^{E-1}q_l^j.
\]
\end{proposition}

\begin{proof}
Set $d_{k,s}=\|u_{k,s}-v_s\|$. Add and subtract $\nabla F_k(v_s)$
inside the difference of the updates. The triangle inequality and
Assumption~\ref{ass:drift} yield
\[
 d_{k,s+1}\leq (1+\eta L_k)d_{k,s}+\eta\delta_k.
\]
Since $d_{k,0}=0$, induction gives~\eqref{eq:local-drift-bound}.
The weighted triangle inequality then proves
\eqref{eq:aggregate-drift-bound}, and the ordinary triangle inequality
through $v_E$ gives the pairwise bound.

For the sharper result, put $b_s=\sum_k\alpha_k u_{k,s}$ and
$r_s=\|b_s-v_s\|$. The identity
$\sum_k\alpha_k\nabla F_k(v_s)=\nabla F(v_s)$ implies
\[
 b_{s+1}-v_{s+1}
 =b_s-v_s-\eta\sum_k\alpha_k
       \big(\nabla F_k(u_{k,s})-\nabla F_k(v_s)\big).
\]
Thus $r_{s+1}\leq r_s+\eta\sum_k\alpha_k L_k d_{k,s}$.
Sum this recurrence from $s=0$ to $E-1$, use $r_0=0$, and substitute
\eqref{eq:local-drift-bound}. This proves
\eqref{eq:aggregate-drift-sharp}; for $E=1$ its right-hand side is zero.
\end{proof}

When $L_k>0$, the right-hand side of
\eqref{eq:local-drift-bound} equals
$\delta_k((1+\eta L_k)^s-1)/L_k$; for $L_k=0$ it is
$s\eta\delta_k$. These bounds do not establish convergence across
communication rounds. They quantify the discrepancy over one block,
and their assumptions must hold along every block to apply repeatedly.

\begin{corollary}[A sufficient distributional condition]
\label{cor:wasserstein-heterogeneity}
Suppose all clients have the same label probabilities $\pi_y$, the
class-conditional input distributions have finite first moments, and
$g_w(x,y)=\nabla_w\ell(f_w(x),y)$ satisfies
\[
 \|g_w(x,y)-g_w(x',y)\|\leq G\|x-x'\|
\]
uniformly over the relevant $w,x,x',y$. Assume that differentiation and
expectation can be interchanged. Let
$P=\sum_k\alpha_kP_k$. Then a valid choice in
Assumption~\ref{ass:drift} is
\begin{equation}
\label{eq:wasserstein-heterogeneity}
 \delta_k=G\sum_{y\in\mathcal Y}\pi_y
 W_1\big(P_k(X\mid Y=y),P(X\mid Y=y)\big).
\end{equation}
Terms with $\pi_y=0$ are omitted. The Wasserstein distance here is
between distributions of the complete input vector.
\end{corollary}

\begin{proof}
For a fixed label, couple $X_k\sim P_k(X\mid Y=y)$ and
$X\sim P(X\mid Y=y)$. The norm of the difference of their expected
sample gradients is at most
$\mathbb E\|g_w(X_k,y)-g_w(X,y)\|\leq G\mathbb E\|X_k-X\|$.
Take the infimum over couplings, then sum with the shared class weights
$\pi_y$ and apply the triangle inequality.
\end{proof}

Smoothness in $w$ alone does not imply the input Lipschitz condition in
Corollary~\ref{cor:wasserstein-heterogeneity}. Further, averaging
distances between individual feature marginals cannot generally replace
the joint-input distance: different dependence structures can give
identical feature marginals.

\subsection{What integrating a fitted path guarantees}
\label{subsec:centroid}

For a given set of received local parameters $w_1,\ldots,w_K$, let
$a=\sum_k\alpha_k w_k$. We consider the continuous piecewise linear
paths
\begin{equation}
\label{eq:corrected-polyline}
 \gamma_k(u)=
 \begin{cases}
  (1-2u)a+2u\theta_k, & 0\leq u\leq\tfrac12,\\
  2(1-u)\theta_k+(2u-1)w_k, & \tfrac12\leq u\leq1.
 \end{cases}
\end{equation}
Both branches equal $\theta_k$ at the midpoint. Their weighted centroid
is available without numerical integration:
\begin{equation}
\label{eq:path-centroid}
 c=\sum_k\alpha_k\int_0^1\gamma_k(u)\,du
   =\sum_k\alpha_k\left(\frac{a+w_k}{4}+\frac{\theta_k}{2}\right).
\end{equation}
Initializing $\theta_k=(a+w_k)/2$ makes every path a straight segment
and gives exactly $c=a$. This is an algebraic FedAvg recovery property;
it does not require a claim about flatness of the loss.

\begin{proposition}[Risk of a path centroid]
\label{prop:centroid-risk}
Let the paths be measurable and square-integrable. Suppose $F$ is
differentiable with $L$-Lipschitz gradient on a convex region containing
their images and $c$. Define
\[
 J=\sum_k\alpha_k\int_0^1 F(\gamma_k(u))\,du,
 \qquad
 V=\sum_k\alpha_k\int_0^1\|\gamma_k(u)-c\|^2\,du,
\]
and assume the displayed quantities are finite. Then, without requiring
convexity of $F$,
\begin{equation}
\label{eq:centroid-risk-bound}
 F(c)\leq J+\frac L2 V.
\end{equation}
\end{proposition}

\begin{proof}
The Lipschitz gradient condition, applied to $-F$ along the segment
from $c$ to any $x$ in the region, gives the lower Taylor bound
\[
 F(x)\geq F(c)+\langle\nabla F(c),x-c\rangle
                    -\frac L2\|x-c\|^2.
\]
Rearrange, set $x=\gamma_k(u)$, and take the weighted integral.
The linear term vanishes because
$\sum_k\alpha_k\int_0^1(\gamma_k(u)-c)\,du=0$.
The remaining terms give~\eqref{eq:centroid-risk-bound}.
\end{proof}

\begin{remark}[Low path risk does not ensure low centroid risk]
\label{rem:centroid-limit}
For $F(w)=(\|w\|^2-1)^2$ on $\mathbb R^2$, every point of the unit
circle has risk zero. The path $\gamma(u)=(\cos(2\pi u),\sin(2\pi u))$
has $J=0$ and centroid $c=0$, where $F(c)=1$. The gradient is Lipschitz
on the closed unit disk, so this example satisfies the local smoothness
assumption of Proposition~\ref{prop:centroid-risk}. The variance term
cannot in general be omitted. In particular, fitting low-risk curves
does not by itself prove that their weight average improves on $a$.
\end{remark}

For a calibration objective $\widehat F$, differentiation under the
integral gives the actual curve-fitting gradient
\begin{equation}
\label{eq:path-fit-gradient}
 \nabla_{\theta_k}\int_0^1\widehat F(\gamma_k(u))\,du
 =\int_0^1\beta(u)\nabla\widehat F(\gamma_k(u))\,du,
 \quad
 \beta(u)=
 \begin{cases}2u,&u\leq\tfrac12,\\2(1-u),&u\geq\tfrac12.
 \end{cases}
\end{equation}
For a finite calibration set and smooth loss this is an ordinary finite
sum followed by a one-dimensional integral; otherwise suitable
domination assumptions are required. If $\widehat F$ is a public labeled
calibration risk, it supplies information beyond the received client
weights. Alternatively, estimating $\nabla F$ through client risk
gradients requires further client queries and communication. Neither
access model is a weights-only guarantee. Agreement between reference
predictions on randomly generated inputs need not estimate either risk.

The term ``path'' has its usual meaning here. No affine connection is
inferred from a fitted curve. For a scalar risk, its derivative along a
smooth path is
$dF(\gamma(u))/du=\langle\nabla F(\gamma(u)),\dot\gamma(u)\rangle$.
This scalar derivative differs from the derivative of the gradient
field; in Euclidean coordinates the latter is
$\nabla^2 F(\gamma(u))\dot\gamma(u)$ when the Hessian exists. No
Euler--Lagrange or covariant-derivative identification is needed for
\eqref{eq:path-fit-gradient}.

\subsection{A conditional descent guarantee}
\label{subsec:conditional-convergence}

\begin{theorem}[Descent under relative gradient error]
\label{thm:inexact-descent}
Suppose $F:\mathbb R^p\to\mathbb R$ has $L$-Lipschitz gradient, with
$L>0$, and an update has the form
\begin{equation}
\label{eq:inexact-update}
 w_{t+1}=w_t-\eta\big(\nabla F(w_t)+e_t\big),
 \qquad \|e_t\|\leq\rho\|\nabla F(w_t)\|,
 \quad 0\leq\rho<1.
\end{equation}
For a constant step size
\begin{equation}
\label{eq:inexact-step-size}
 0<\eta\leq\frac{1-\rho}{L(1+\rho)^2},
\end{equation}
every update satisfies
\begin{equation}
\label{eq:inexact-descent}
 F(w_{t+1})\leq F(w_t)
       -\frac{\eta(1-\rho)}2\|\nabla F(w_t)\|^2.
\end{equation}
If $F\geq F_{\inf}>-\infty$, then
\begin{equation}
\label{eq:stationarity-rate}
 \min_{0\leq t<T}\|\nabla F(w_t)\|^2
 \leq \frac{2\big(F(w_0)-F_{\inf}\big)}
              {\eta(1-\rho)T},
\end{equation}
and $\|\nabla F(w_t)\|\to0$. This conclusion does not assert convergence
of the parameter sequence to a particular minimizer.
\end{theorem}

\begin{proof}
Write $g=\nabla F(w_t)$ and $e=e_t$. Smoothness gives
\[
 F(w_{t+1})\leq F(w_t)-\eta\langle g,g+e\rangle
                       +\frac{L\eta^2}{2}\|g+e\|^2.
\]
The relative-error assumption implies
$\langle g,g+e\rangle\geq(1-\rho)\|g\|^2$ and
$\|g+e\|\leq(1+\rho)\|g\|$. Hence
\[
 F(w_{t+1})\leq F(w_t)
 -\eta\left((1-\rho)-\frac{L\eta}{2}(1+\rho)^2\right)\|g\|^2.
\]
The step-size bound makes the coefficient in parentheses at least
$(1-\rho)/2$, proving~\eqref{eq:inexact-descent}. Sum this inequality
over $t=0,\ldots,T-1$ and use the lower bound on $F$ to obtain
\eqref{eq:stationarity-rate}. The same telescoping argument bounds the
infinite sum of the nonnegative terms $\|\nabla F(w_t)\|^2$, so those
terms tend to zero.
\end{proof}

For an arbitrary candidate aggregation $c_t$, one can define
$e_t=(w_t-c_t)/\eta-\nabla F(w_t)$, but the relative-error inequality in
\eqref{eq:inexact-update} must then be proved or verified separately.
The path construction, the drift bound, and a comparison of $c_t$ with
the FedAvg anchor $a_t$ do not supply that inequality. A guard accepting
$c_t$ only when its validation loss is no larger than that of $a_t$
guarantees only the stated comparison on that validation set. It does
not establish a comparison with $w_t$, population-risk descent,
test-risk improvement, or unconditional convergence of the proposed
aggregation procedure.

\section{Reference extension: affine adapters and a descent safeguard}
\label{sec:safeguarded}

The natural-domain experiments below add an explicit input adapter and
replace the validation-only guard in Equation~\eqref{eq:guard} with a
distributed training-risk safeguard. The latter admits a convergence
statement for the implemented extension. It requires substantially more
client access than the public-data procedure of Section~\ref{sec:algorithm}.

\subsection{What an affine adapter can establish}

Let a canonical pair $(Z,Y)$ have joint distribution $P_0$, and suppose
client $k$ observes $X_k=A_kZ+b_k$, with $A_k$ invertible. With the known
adapter $S_k(x)=A_k^{-1}(x-b_k)$, every client has the same adapted joint
distribution: $(S_k(X_k),Y)\sim P_0$. Therefore
\begin{equation}\label{eq:adapter-risk}
 \mathbb E\ell(f_w(S_k(X_k)),Y)
 =\mathbb E\ell(f_w(Z),Y)
\end{equation}
for every $w$ with finite risk. Equivalently, for a model whose first
affine layer is $Bz+d$, its client-coordinate layer is
$B_k=BA_k^{-1}$, $d_k=d-BA_k^{-1}b_k$; substituting $X_k$ gives exactly
$BZ+d$. This is an explicit compatible model transformation, rather
than an inference from a fitted weight-space path.

There is also a quantitative, conditional statement for approximate
adapters. Suppose
$g_w(x,y)=\nabla_w\ell(f_w(x),y)$ is $G$-Lipschitz in $x$, uniformly
over the parameters under consideration, and differentiation and
expectation can be interchanged. Define
\[
 \varepsilon_k=\mathbb E\|S_k(A_kZ+b_k)-Z\|,
 \qquad \bar\varepsilon=\sum_k\alpha_k\varepsilon_k.
\]
Writing $F_k^S$ for the adapted risk, $F^S=\sum_k\alpha_k F_k^S$,
and $F_0$ for the canonical risk, the triangle inequality gives
\begin{equation}\label{eq:adapter-gradient-error}
 \|\nabla F_k^S(w)-\nabla F_0(w)\|\leq G\varepsilon_k,
 \qquad
 \|\nabla F_k^S(w)-\nabla F^S(w)\|
 \leq G(\varepsilon_k+\bar\varepsilon).
\end{equation}
Indeed, couple each adapted observation with the same latent $(Z,Y)$,
bound the sample-gradient difference by its input distance, and then
average. This provides one sufficient mechanism for reducing the
heterogeneity term in Proposition~\ref{prop:local-drift}.

The natural subject and acquisition-batch domains need not satisfy this
common-latent model or preserve label proportions. Their implemented
adapters are diagonal affine standardizations estimated from each
client's training features and frozen before optimization. Matching
means and variances does not identify an arbitrary $A_k,b_k$, align
class-conditional distributions, or establish
Equation~\eqref{eq:adapter-gradient-error}. In particular, different
class proportions can change these moments even without a change of
measurement coordinates. We evaluate these adapters as an ablation,
not as recovered physical transformations. Unseen-domain inference uses
a predeclared source-only adapter; test features do not estimate it.

\subsection{An empirical objective with a certified fallback}

For fixed adapters and finite client training sets $D_k$, use the
regularized empirical objective
\begin{equation}\label{eq:safe-objective}
 F(w)=\sum_k\alpha_k\frac1{n_k}
       \sum_{(x,y)\in D_k}\ell(f_w(S_kx),y)
       +\frac\lambda2\|w\|^2,
 \qquad \alpha_k=\frac{n_k}{\sum_jn_j},\quad\lambda>0.
\end{equation}
All trainable parameters, including biases, enter the penalty. The
experiments use smooth tanh networks and cross-entropy. This section's
$F$ is the fixed training objective in~\eqref{eq:safe-objective}, not
held-out risk or the public calibration objective.

At round $t$, clients compute the exact full-training gradient
$g_t=\nabla F(w_t)$ by weighted aggregation. Starting at the same
$\eta_0>0$ each round, reduce the trial step by a fixed factor
$\tau\in(0,1)$ until the fallback $b_t=w_t-\eta_tg_t$ satisfies
\begin{equation}\label{eq:safe-armijo}
 F(b_t)\leq F(w_t)-\frac{\eta_t}{2}\|g_t\|^2.
\end{equation}
When $g_t=0$, take $b_t=w_t$. Let $a_t$ be the local-training FedAvg
candidate and $c_t$ the calibration-fitted curve centroid. Each client
then evaluates its exact training risk at each candidate. Set
\begin{equation}\label{eq:safe-selection}
 w_{t+1}\in\argmin_{v\in\{b_t,a_t,c_t\}} F(v).
\end{equation}
The fitting procedure can be stochastic and need not satisfy a
relative-gradient-error bound: only the fallback and candidate
evaluations certify descent. For the no-curve safeguard ablation use
$\{b_t,a_t\}$; for the calibration-matched safeguard use
$\{b_t,a_t,m_t\}$, where $m_t$ is the calibration-trained candidate.
The same argument applies to both.

\begin{theorem}[Safeguarded empirical stationarity]
\label{thm:safeguarded}
Suppose $F$ is twice continuously differentiable and coercive on
$\mathbb R^p$, and bounded below by $F_{\inf}>-\infty$. Adapters,
training sets, weights, and regularization are fixed. With exact
gradients, losses, and arithmetic, the backtracking procedure above
terminates each round. There exists a finite constant $L>0$, independent
of $t$, such that, with
$\eta_{\min}=\min\{\eta_0,\tau/L\}>0$,
\begin{equation}\label{eq:safe-descent}
 F(w_{t+1})\leq F(w_t)-\frac{\eta_t}{2}\|\nabla F(w_t)\|^2,
 \qquad \eta_t\geq\eta_{\min}.
\end{equation}
Consequently,
\begin{equation}\label{eq:safe-rate}
 \min_{0\leq t<T}\|\nabla F(w_t)\|^2
 \leq\frac{2(F(w_0)-F_{\inf})}{\eta_{\min}T},
 \qquad \|\nabla F(w_t)\|\longrightarrow0.
\end{equation}
The iterates have accumulation points, and every accumulation point is
stationary. Convergence to a global minimizer is not asserted.
\end{theorem}

\begin{proof}
Let $\mathcal K=\{w:F(w)\leq F(w_0)\}$, compact by continuity and
coercivity, and let $M=\max_{w\in\mathcal K}\|\nabla F(w)\|<\infty$.
Take a closed ball containing $\mathcal K$ and every point within
distance $\eta_0M$ of it. The continuous Hessian is bounded in operator
norm on this ball by some finite $L>0$. Whenever $w_t\in\mathcal K$,
every fallback trial segment lies in the ball, because its step size
is at most $\eta_0$. The descent lemma on this segment yields
\[
 F(w_t-\eta g_t)
 \leq F(w_t)-\eta(1-L\eta/2)\|g_t\|^2.
\]
Thus every $\eta\leq1/L$ passes~\eqref{eq:safe-armijo}. Backtracking
terminates; if its initial trial passes, $\eta_t=\eta_0$, and otherwise
the immediately preceding failed trial must exceed $1/L$, giving
$\eta_t>\tau/L$. Selection~\eqref{eq:safe-selection} proves
\eqref{eq:safe-descent} and keeps the next iterate in $\mathcal K$.
Induction therefore justifies the same compact set and constant at
every round. Telescope~\eqref{eq:safe-descent} to obtain
$\sum_{t=0}^{T-1}\|\nabla F(w_t)\|^2
\leq2(F(w_0)-F_{\inf})/\eta_{\min}$. This proves the displayed rate
and, as $T\to\infty$, summability and convergence of the squared
gradient norms to zero. Compactness supplies accumulation points;
continuity of the gradient makes them stationary.
\end{proof}

For the implemented finite smooth network, softmax cross-entropy is
nonnegative and twice continuously differentiable in all parameters.
The positive full-parameter quadratic penalty therefore makes
\eqref{eq:safe-objective} coercive and lower bounded, satisfying the
theorem's objective assumptions. The result covers arbitrary fixed
non-IID training sets; affine alignment is not required for the
safeguard's convergence. With a fixed feature map and a linear softmax
head, this objective is additionally $\lambda$-strongly convex. In that
special case the same proof and
$\|\nabla F(w)\|^2\geq2\lambda(F(w)-F^*)$ give
$F(w_T)-F^*\leq(1-\lambda\eta_{\min})^T(F(w_0)-F^*)$ and a unique
global minimizer. This stronger conclusion does not apply to the
trained nonlinear networks.

The safeguard adds full-client gradient passes, candidate-loss passes,
and line-search queries; these are reported separately from ordinary
FedAvg costs. It guarantees descent of the regularized training
objective, not test accuracy, population risk, privacy, or superiority
over a separately trained baseline. At a shared starting point it can
choose a candidate with no greater training objective than FedAvg's;
separate trajectories do not inherit that comparison. Finite-precision
Armijo residuals are logged. A fixed numerical tolerance changes an
exact descent statement by that tolerance per round; it does not by
itself prove asymptotic stationarity. The code raises an error if its
finite line-search budget cannot obtain its recorded certificate.

\section{Local curves with a selected shared endpoint}
\label{sec:local-curves}

We now replace the public calibration access model and curve-centroid
proposal with a distributed construction. Each client fits a path from its
own local solution to a common candidate, using only its own training
records. The returned model is a selected shared endpoint. This tests a
different mechanism from the centroid rules above: locally unconstrained
parameters may permit a client to incorporate another domain without
sacrificing its existing fit. Redundancy does not establish that such a
compatible solution or a low-loss connecting path exists. Intersections of
client-specific mode connections have already been studied by
\citet{thakur2025clinical}; we claim no novelty for that broad idea.

\subsection{Local objectives, private controls, and a common endpoint}

Here $L_k(w)$ denotes the \emph{fixed regularized empirical} training loss
of client $k$, with the same full-parameter quadratic penalty as
Equation~\eqref{eq:safe-objective}; $L(w)=\sum_k\alpha_kL_k(w)$.
Input preprocessing is fixed before optimization. No calibration or test
records enter the fitting or selection objective. Suppress the round index.
After local training from $w_t$, client $k$ retains $v_k$. An aggregation
interface forms the FedAvg anchor $a=\sum_k\alpha_kv_k$ without requiring
individual $v_k$ to be disclosed. Set $z^0=a$ and initialize a local control
$c_k^0=(v_k+a)/2$. The quadratic B\'ezier path is
\begin{equation}\label{eq:local-bezier}
 \gamma_k(s;c_k,z)=(1-s)^2v_k+2s(1-s)c_k+s^2z,
 \qquad 0\leq s\leq1.
\end{equation}
Its endpoints are $v_k$ and $z$; controls and local endpoints remain on
clients in the intended protocol. Fit using the local surrogate
\begin{equation}\label{eq:local-curve-surrogate}
 Q_k(c_k,z)=L_k(z)
 +\frac{\beta}{|\mathcal S|}\sum_{s\in\mathcal S}
       [L_k(\gamma_k(s))-L_k(v_k)-\varepsilon]_+^2
 +\frac\rho2\|c_k-(v_k+z)/2\|^2,
\end{equation}
where $[r]_+=\max\{r,0\}$. The final term discourages excessive bending
away from the current straight-line midpoint; it depends on both $c_k$ and $z$.
The endpoint loss is explicit: an average path penalty alone does not
certify the eventual model. The implemented defaults are
$\mathcal S=\{1/4,1/2,3/4,1\}$, $\beta=1$, $\rho=10^{-3}$,
$\varepsilon=0.02$, and three fitting steps with learning rate $0.04$
for both controls and endpoint. These tolerances apply to the stated
regularized training loss; they are not accuracy guarantees.

Write $r_{k,s}=[L_k(\gamma_k(s))-L_k(v_k)-\varepsilon]_+$ and
$d_k=c_k-(v_k+z)/2$. Differentiation gives
\begin{align}
 \nabla_zQ_k&=\nabla L_k(z)-\rho d_k/2
 +\frac{2\beta}{|\mathcal S|}\sum_s r_{k,s}s^2
                      \nabla L_k(\gamma_k(s)),\label{eq:local-z-gradient}\\
 \nabla_{c_k}Q_k&=\frac{2\beta}{|\mathcal S|}\sum_s
 r_{k,s}\,2s(1-s)\nabla L_k(\gamma_k(s))
 +\rho d_k.\label{eq:local-c-gradient}
\end{align}
Clients update controls locally and contribute endpoint gradients to the
weighted aggregate $\sum_k\alpha_k\nabla_zQ_k$. Thus the common endpoint
is optimized in the full parameter space; this is distributed optimization
with a path penalty, rather than a one-dimensional search on a previously
fixed curve. Local control steps on $Q_k$ and weighted endpoint steps on
$\sum_k\alpha_kQ_k$ amount to block-scaled gradient updates. No global
optimality or feasibility guarantee follows from finitely many such steps.

\subsection{Selection and the limits of compatibility certificates}

Save the fitted endpoints and select with distributed training-loss queries:
\begin{equation}\label{eq:local-selection}
 \mathcal C_t=\{w_t,a,z^1,\ldots,z^J\},\qquad
 w_{t+1}\in\argmin_{u\in\mathcal C_t}L(u).
\end{equation}
With exact evaluations, $L(w_{t+1})\leq\min\{L(w_t),L(a)\}$. This is
monotone training loss and a comparison with the current round's FedAvg
candidate, not domination of a separate FedAvg trajectory. Inclusion of
$w_t$ can produce repeated unchanged iterates. Without the aggregate-gradient
fallback of Section~\ref{sec:safeguarded}, it does not establish
Theorem~\ref{thm:safeguarded}'s stationarity conclusion. The path penalty
and weighted selection also do not guarantee preservation of every client.

If the selected endpoint actually satisfies
$L_k(z)\leq L_k(v_k)+\varepsilon$ for every $k$, then its weighted loss
satisfies $L(z)\leq\sum_k\alpha_kL_k(v_k)+\varepsilon$.
This is a conditional compatibility certificate, not a consequence of
small surrogate loss alone. Even satisfaction at every sampled point does
not certify the entire path. If $s\mapsto L_k(\gamma_k(s))$ is
$B_k$-Lipschitz and all grid points, including $0$ and $1$, satisfy the
loss tolerance, then the nearest-grid-point argument gives
\begin{equation}\label{eq:local-grid-certificate}
 \max_{s\in[0,1]}L_k(\gamma_k(s))
 \leq L_k(v_k)+\varepsilon+B_kh/2,
\end{equation}
where $h$ is the largest grid gap. The implementation does not certify a
numerical $B_k$ and reports sampled violations rather than asserting
whole-path feasibility. No population-risk statement follows from these
finite training inequalities.

For noisy loss queries, if every candidate estimate obeys
$|\widetilde L(u)-L(u)|\leq\delta$, selecting its estimated minimizer gives
$L(w_{t+1})\leq\min_{u\in\mathcal C_t}L(u)+2\delta$.
Indeed, the selected estimate is no larger than the best candidate's
estimate, and converting each estimate to its true value costs at most
$\delta$. Requiring an estimated improvement of at least $2\delta$ over a
specified baseline certifies nonincrease relative to that baseline on the
same error event. Privacy noise therefore changes the exact monotonicity
claim; repeated adaptive queries also require an appropriate privacy and
statistical analysis.

\subsection{Data locality and matched controls}

The protocol needs local training records, broadcasts of shared candidates,
aggregated model contributions, endpoint gradients, and scalar loss queries.
It does not require client-to-client record sharing or a server-held labeled
calibration set. The main accuracy experiments emulate clients in one process
using exact aggregate-only interfaces; those runs use neither cryptographic
aggregation nor differential privacy. This is a data-flow design and
numerical evaluation, not an end-to-end privacy guarantee. Model
aggregates, gradients, losses, sample counts, preprocessing summaries, and
any released compatibility diagnostics can still reveal information.
A protected deployment needs an explicit participation and collusion model,
secure aggregation, and, for a formal privacy claim, bounded contributions,
noise introduced before an untrusted server observes them, and accounting
for all fitting and selection releases~\citep{kairouz2021distributed}.

Pairwise masking is a component of secure aggregation\citep{bonawitz2017secure}.
Our optional restricted prototype masks fixed-point contributions pairwise. It covers
model, gradient and loss contributions; source-moment preprocessing
still uses ordinary simulated sums and is not protected by this backend.
Ephemeral
client X25519 keys, HKDF-SHA256, and query-specific
SHAKE256 masks protect fixed-point contributions before the aggregation
interface receives them. It assumes an honest-but-curious server, an
externally authenticated common key roster, at least three clients and two
noncolluding clients. Dropout aborts the session. Malicious-party protection,
key authentication, network transport, process isolation, and differential
privacy are not implemented. This small prototype is not an audited
deployment protocol. With $K$ bounded contributions and scale $M$, its
per-coordinate quantization error is at most $K/(2M)$ under the enforced
no-overflow conditions. Its selector therefore uses the approximate-selection
allowance above, not an exact monotonicity assertion. A three-client
integration check has largest fixed-candidate loss error
$2.69\times10^{-8}$, below its bound $8.94\times10^{-8}$; this is numerical
verification, not a proof of privacy. The main accuracy table uses exact sums.

The direct-training control starts from the same anchor and uses
$J(|\mathcal S|+1)$ distributed gradient steps, matching endpoint and path
gradient evaluations and training examples processed. Its checkpoints are
saved at $J$ corresponding budget boundaries and selected with the same
old-model and anchor candidates. This matches gradient work, while
communication and scalar-loss query costs remain separately measurable.
A centroid ablation uses the same fitting procedure on its own trajectory
but proposes only
\begin{equation}\label{eq:local-centroid-ablation}
 \bar z=\sum_k\alpha_k\int_0^1\gamma_k(s)\,ds
       =\sum_k\alpha_k\frac{v_k+c_k+z^J}{3},
\end{equation}
then selects among $\{w_t,a,\bar z\}$. Consequently it receives fewer
proposal candidates; this comparison measures the specified aggregation
rules rather than isolating only their parameter locations. These controls
are necessary to distinguish benefits of curves from additional distributed
optimization and model-selection opportunities.

\section{Evaluation without server calibration data}
\label{sec:local-curve-experiments}

The current method, \textbf{CovFedAvg}, fits local curves and selects a
shared endpoint using client training data only. It is distinct from the
public-calibration variants retained in Section~\ref{sec:natural-experiments}.
All accuracy-benchmark runs use an exact aggregate interface in one process;
they evaluate optimization and data flow, not deployment security.

\subsection{Complementary support: an identifiable motivating example}

Consider $y=2|x|+0.3$ with 40 magnitudes uniformly spaced in $[0.1,1]$.
Client A sees positive inputs; B sees their negatives. The label multisets
are identical. A linear model of the feature $|x|$ would have the same
identifiable parameters on either client and would not demonstrate the
proposed failure. Instead use the redundant shared model
\[
 f_{u,v,b}(x)=u\operatorname{ReLU}(x)+v\operatorname{ReLU}(-x)+b.
\]
Local fits from zero inactive branches approach $(2,0,0.3)$ and $(0,2,0.3)$,
whose average $(1,1,0.3)$ loses half of each learned slope. The compatible
solution is $(2,2,0.3)$. Client A can vary $v$ without changing its predictions;
client B can vary $u$. Their local zero-loss sets intersect at the shared
solution. This is a support-coverage mechanism, not an invertible
transformation assumption or a consequence of equal label frequencies.

The numerical diagnostic uses 1,800 local full-batch updates at learning
rate $0.3$, followed by 1,100 path-fitting steps at $0.3$, with $\rho=0$ and
$\varepsilon=0$. These illustrative budgets differ from the small fixed
benchmark budget below. One-merge FedAvg has training MSE $0.37346$; the
centroid of the same final fitted paths has MSE $0.12400$; selecting the
shared endpoint gives $3.24\times10^{-9}$. Matched direct distributed
fine-tuning also solves this convex task, with MSE $2.13\times10^{-30}$.
Both fitting procedures use 11,000 client full-gradient evaluations.
This example establishes feasibility, not an advantage over that control.
Repeated FedAvg can also recover the solution in this example; the
one-merge failure is not an asymptotic impossibility statement.

The soft path penalty is not an exact constraint solver. Even though the
selected endpoint is nearly correct, a 101-node path diagnostic has weighted
maximum violation $0.00361$ at zero tolerance. We retain this violation
rather than asserting exact preservation along the whole curve.

\begin{figure}[tbp]
\centering
\includegraphics[width=\linewidth]{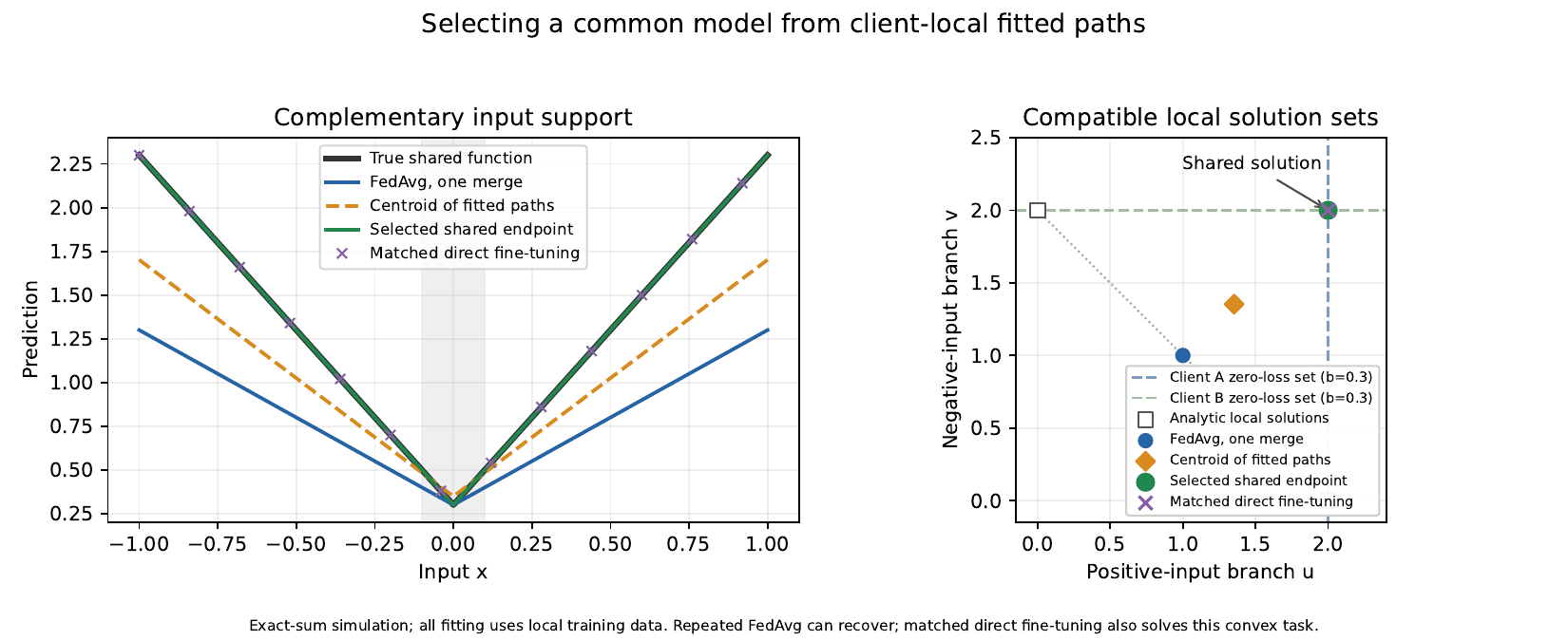}
\caption{Complementary input support and selection of a shared model.
The selected endpoint recovers both branches; averaging the same fitted
paths does not. Matched direct fine-tuning also recovers the solution.
The right panel projects parameters onto $(u,v)$; its analytic zero-loss
lines assume $b=0.3$, whereas fitted methods can have slightly different
intercepts. The dotted segment connects the analytic local solutions and
does not depict a learned low-loss curve.}
\label{fig:local-abs-example}
\end{figure}

\subsection{Fixed protocol and controls}

The broader study has four tasks, five seeds $(11,29,47,71,89)$, hidden
widths 16 and 64, and 20 rounds with five local epochs. The architecture
is the same regularized tanh MLP for every method within a task and width.
Batch size is 64, local learning rate is $0.04$, and weight decay is $10^{-4}$.
All methods share initialization and local shuffles. The path parameters
are fixed as specified in Section~\ref{sec:local-curves}, before benchmark
test evaluation. No test-driven hyperparameter search is performed.

The synthetic support task uses a fixed, noisy, three-class tanh teacher
in six input dimensions and six clients with disjoint intervals of the
first coordinate. The conditional labeling rule is shared; label frequencies
may change because the feature support changes. It has 1,536 training and
1,536 test examples. It is a controlled covariate-shift diagnostic, not
a naturally collected dataset.

The other tasks retain the recorded UCI HAR subject identities, Gas
acquisition batches, and future-batch Gas split of
Section~\ref{sec:natural-experiments}. The same 5,066, 8,330, and 7,861
source training records and the same test records are retained.
Previously reserved calibration and validation records are \emph{unused},
not reassigned to the server or clients. A source-global scaler is computed
from aggregated training counts, sums, and squared sums and then frozen.
No individual feature moments are required by the server-facing API;
sample counts and sample-based client weights are public in this simulation.
The publisher's upstream preprocessing remains outside our provenance audit.

Five methods are compared: FedAvg; FedAvg with selection between the old
model and the new average; CovFedAvg; a same-procedure curve-centroid
ablation; and direct distributed fine-tuning. Each curve-fitting step uses
one endpoint and four path gradients per client. The matched direct control
therefore takes 15 full-client gradient steps per round, saving checkpoints
after steps 5, 10, and 15, and uses the same five-candidate selection budget
as CovFedAvg. This matches gradient evaluations and processed examples,
not communication: direct fine-tuning communicates 15 gradient updates,
whereas the curve method communicates three endpoint-gradient updates.
The centroid ablation uses the same fitting procedure but selects only
among its centroid, FedAvg, and the old model. It has fewer selection
opportunities, and the methods follow different trajectories. Its comparison
therefore does not isolate only the averaging operation.

Training-risk reporting and sampled-path diagnostics are accounted for
separately from algorithm queries. Test evaluation occurs only after the
last round. The 200 method runs produce 4,000 round records. Raw metrics,
costs, source-only preprocessing, configuration, split provenance, and
checksums are provided in \texttt{results/local-curves/}.

\subsection{Results and limits}

CovFedAvg improves mean macro balanced accuracy over FedAvg in all eight
task--width settings, by $0.14$--$1.12$ percentage points. It improves in
35 of the 40 paired seed runs. Relative to the curve-centroid ablation,
mean improvements range from $0.07$ to $0.76$ percentage points, with
31 of 40 paired wins. These are descriptive comparisons on the recorded
benchmarks, not a statistical significance claim or independent replications
across populations. The gains are modest and largest on participating Gas.

At width 16, macro domain balanced accuracy is $91.58\%$ for FedAvg and
$91.73\%$ for CovFedAvg on HAR, compared with $92.02\%$ for matched direct
fine-tuning. On participating Gas batches the corresponding values are
$88.88\%$, $89.84\%$, and $91.65\%$. At width 64, Gas obtains $89.07\%$,
$90.19\%$, and $92.89\%$, respectively. Increased width does not produce
a consistent improvement across tasks. On future Gas, width-16 accuracy
changes from $76.51\%$ to $76.98\%$, close to the direct control's $77.01\%$.
The reported seed standard deviations and worst-domain scores are given
in Table~\ref{tab:local-curve-results}; these are repeated runs on fixed
benchmark populations, not independent institutions.

All 3,200 selected-method rounds satisfy nonincrease of the evaluated
training objective and no greater loss than their own current-round
FedAvg candidate. These checks validate the selector's empirical comparison,
not stationarity or test improvement. In particular, average improvement
does not imply that every client is preserved: on HAR the worst-domain
balanced accuracy can decrease. The code records endpoint and sampled-path
violation summaries; finite soft-penalty optimization does not guarantee
that every client's tolerance is met.

The mechanism works in the motivating example and produces modest gains
over FedAvg on participating Gas. It does not establish a reliable advantage
over direct distributed optimization with the same gradient-example budget.
The matched direct control has higher mean balanced accuracy in all eight
settings, while requiring 15 gradient-communication updates per outer round
instead of three endpoint-gradient updates. This is a measured accuracy and
communication tradeoff; it does not imply fivefold lower total communication
because common local training, selection, and other messages also contribute.
Further claims about geometry or redundancy require stronger evidence.

\begin{table}[p]
\centering\small
\caption{Local-data study: mean $\pm$ sample standard deviation over five
seeds. BAcc. averages recalls over classes present in each test domain;
macro averages domains and worst takes the smallest domain score. No
calibration records are used. All methods use the same training data within
each task.}
\label{tab:local-curve-results}
\resizebox{\textwidth}{!}{\begin{tabular}{llrr}
\toprule
Domain / width & Method & Macro BAcc. (\%) & Worst BAcc. (\%) \\
\midrule
support shift / 16 & FedAvg & $73.13\pm1.90$ & $59.63\pm1.29$ \\
support shift / 16 & FedAvg + selection & $73.13\pm1.90$ & $59.63\pm1.29$ \\
support shift / 16 & Matched distributed fine-tuning & $74.11\pm1.70$ & $60.41\pm1.37$ \\
support shift / 16 & CovFedAvg & $73.65\pm2.08$ & $59.76\pm1.18$ \\
support shift / 16 & Same-fit curve centroid & $73.49\pm1.96$ & $59.67\pm1.24$ \\
support shift / 64 & FedAvg & $73.29\pm1.85$ & $60.03\pm0.83$ \\
support shift / 64 & FedAvg + selection & $73.29\pm1.85$ & $60.03\pm0.83$ \\
support shift / 64 & Matched distributed fine-tuning & $73.82\pm1.76$ & $60.97\pm0.80$ \\
support shift / 64 & CovFedAvg & $73.57\pm1.94$ & $60.29\pm0.75$ \\
support shift / 64 & Same-fit curve centroid & $73.30\pm1.89$ & $60.12\pm0.79$ \\
har / 16 & FedAvg & $91.58\pm0.60$ & $80.19\pm2.05$ \\
har / 16 & FedAvg + selection & $91.58\pm0.60$ & $80.19\pm2.05$ \\
har / 16 & Matched distributed fine-tuning & $92.02\pm0.46$ & $79.04\pm2.87$ \\
har / 16 & CovFedAvg & $91.73\pm0.65$ & $79.54\pm2.70$ \\
har / 16 & Same-fit curve centroid & $91.66\pm0.57$ & $80.17\pm2.26$ \\
har / 64 & FedAvg & $91.51\pm0.63$ & $78.59\pm2.13$ \\
har / 64 & FedAvg + selection & $91.51\pm0.63$ & $78.59\pm2.13$ \\
har / 64 & Matched distributed fine-tuning & $91.91\pm0.37$ & $77.61\pm2.31$ \\
har / 64 & CovFedAvg & $91.65\pm0.52$ & $78.02\pm2.42$ \\
har / 64 & Same-fit curve centroid & $91.58\pm0.63$ & $78.70\pm2.34$ \\
gas / 16 & FedAvg & $88.88\pm0.58$ & $56.59\pm2.05$ \\
gas / 16 & FedAvg + selection & $88.88\pm0.58$ & $56.59\pm2.05$ \\
gas / 16 & Matched distributed fine-tuning & $91.65\pm1.04$ & $64.49\pm3.09$ \\
gas / 16 & CovFedAvg & $89.84\pm0.93$ & $59.01\pm1.95$ \\
gas / 16 & Same-fit curve centroid & $89.08\pm0.52$ & $57.33\pm1.88$ \\
gas / 64 & FedAvg & $89.07\pm0.33$ & $55.49\pm1.83$ \\
gas / 64 & FedAvg + selection & $89.23\pm0.45$ & $55.49\pm1.83$ \\
gas / 64 & Matched distributed fine-tuning & $92.89\pm0.91$ & $67.69\pm7.36$ \\
gas / 64 & CovFedAvg & $90.19\pm0.73$ & $57.53\pm1.76$ \\
gas / 64 & Same-fit curve centroid & $89.54\pm0.44$ & $56.61\pm2.20$ \\
gas future / 16 & FedAvg & $76.51\pm1.69$ & $71.61\pm3.38$ \\
gas future / 16 & FedAvg + selection & $76.58\pm1.69$ & $71.70\pm3.37$ \\
gas future / 16 & Matched distributed fine-tuning & $77.01\pm1.83$ & $72.75\pm3.73$ \\
gas future / 16 & CovFedAvg & $76.98\pm1.63$ & $72.54\pm3.28$ \\
gas future / 16 & Same-fit curve centroid & $76.74\pm1.73$ & $72.07\pm3.47$ \\
gas future / 64 & FedAvg & $76.97\pm1.86$ & $72.18\pm3.74$ \\
gas future / 64 & FedAvg + selection & $76.92\pm1.82$ & $72.08\pm3.67$ \\
gas future / 64 & Matched distributed fine-tuning & $77.64\pm1.32$ & $73.77\pm2.68$ \\
gas future / 64 & CovFedAvg & $77.60\pm1.49$ & $73.51\pm3.07$ \\
gas future / 64 & Same-fit curve centroid & $77.31\pm1.69$ & $72.91\pm3.46$ \\
\bottomrule
\end{tabular}
}
\end{table}

\begin{figure}[p]
\centering
\includegraphics[width=\textwidth]{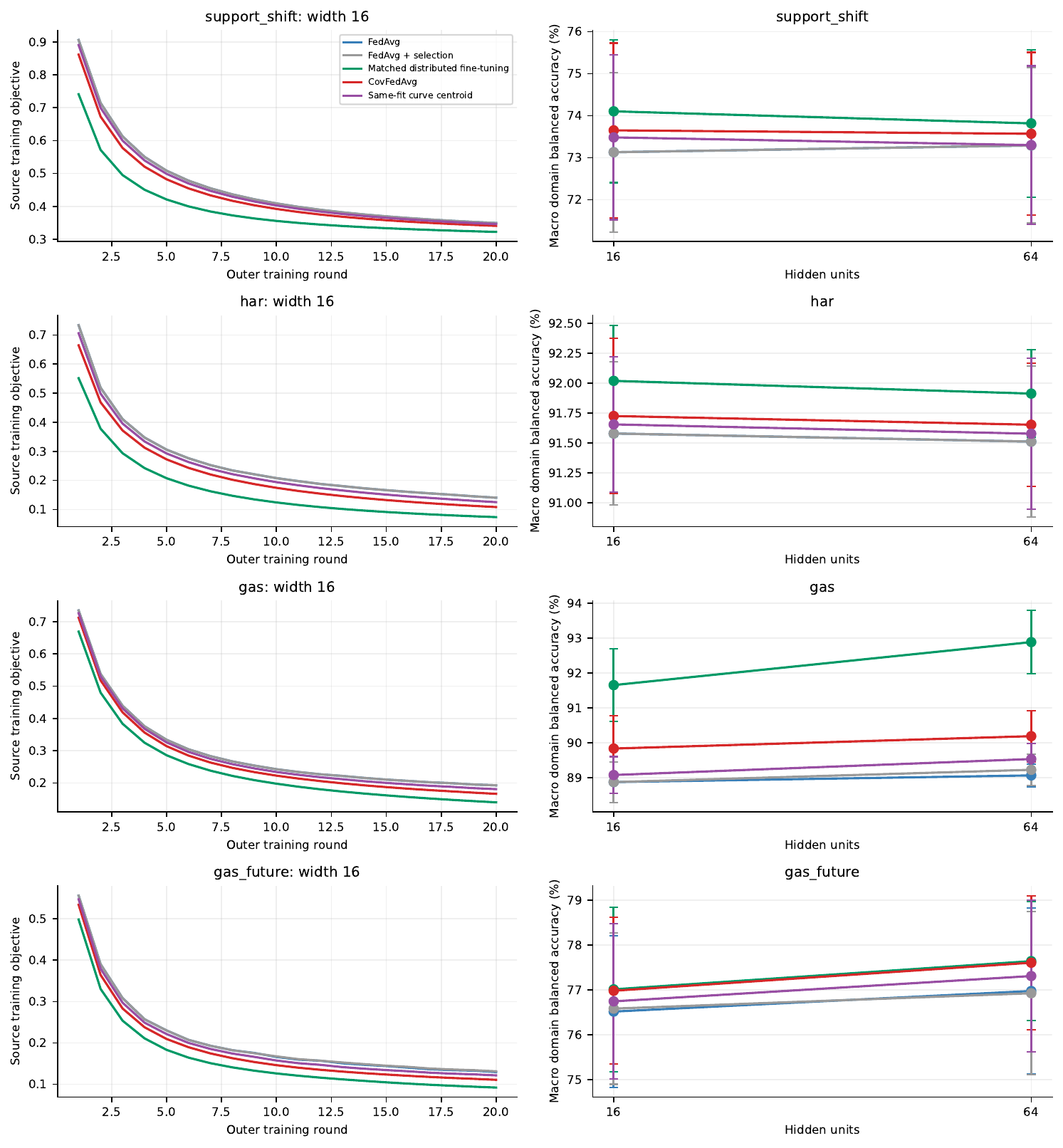}
\caption{CovFedAvg and controls without centralized calibration.
Left: mean source training objective at width 16. Right: final macro domain
balanced accuracy versus width, with sample standard deviation over five
seeds. Each outer round includes method-specific additional queries,
whose costs are reported separately. Source-objective selection does not guarantee improved held-out
accuracy or preservation of each client.}
\label{fig:local-curves-evaluation}
\end{figure}

\FloatBarrier
\section{Reference evaluation of public-calibration variants}
\label{sec:natural-experiments}

This earlier evaluation uses recorded subject and acquisition-batch identities,
without author-generated feature transformations or label-skew partitions.
The questions are whether native heterogeneity affects optimization, whether
frozen affine standardizers help, and whether a curve candidate contributes
beyond the exact safeguard and its additional calibration labels. Non-IID
data do not imply that FedAvg must perform poorly. We retain strong baseline
results and separate heterogeneity diagnostics from the observed method
ranking. The controlled-shift pilot is retained in
Section~\ref{sec:experiments} as supplementary feasibility evidence.

\subsection{Tasks, native domains, and split provenance}

\paragraph{Wearable activity recognition.}
UCI Human Activity Recognition Using Smartphones (HAR) provides 10,299
observations from 30 people, six activities, 561 published time/frequency
features, and recorded subject IDs\cite{reyesortiz2013har}. Each source-training
subject is a client. We sort the 21 official training subject IDs and assign
the first 15 to client training, the next three to public calibration, and
the last three to an unused validation pool. All nine official test subjects
remain untouched and unseen during training. The resulting counts are
5,066 training, 1,177 reserved calibration, 1,109 unused validation, and
2,947 test observations. These roles are fixed across seeds. The publisher's
windows overlap by 50\%; assigning complete subjects to one role prevents
neighboring windows of one person from crossing roles. Features arrive in
the publisher's normalized representation, whose original normalization fit
we cannot independently audit. Every additional transformation in our code
is estimated from our training subset only. This is a subject-shift task
using one documented smartphone setup, rather than a multi-device study.

\paragraph{Gas classification across acquisition periods.}
UCI Gas Sensor Array Drift supplies 13,910 observations, 128 response
features, six gases, and ten recorded acquisition batches from the same
16-sensor array\cite{vergara2012gas}. In the participating-domain task, every
native batch is a client. Within each batch and observed class, independently
shuffle records and allocate 60\%/10\%/10\%/20\% to training/calibration/unused
validation/test using cumulative integer boundaries. Realized totals are
8,330/1,388/1,375/2,817. Missing classes and unequal client sizes are retained;
the task is not made homogeneous by balancing the native domains. This
within-batch holdout measures prediction on participating temporal domains,
not forecasting into a new acquisition period. Finer acquisition timestamps
are unavailable, so residual within-batch dependence is possible.

A secondary future-domain task trains on batches 1--8, with within-batch
80\%/10\%/10\% training/calibration/unused validation partitions, and tests on
all records from batches 9--10. The respective totals are
7,861/980/999/4,070. No held-out batch features or labels estimate a
transformation or choose an update. Some publisher batches combine
nonconsecutive acquisition months; we preserve their native IDs instead of
inventing finer timestamps. These domains represent temporal changes of one
array, rather than ten independent physical devices.

All archive digests, source URLs, row identifiers, class histograms, native
IDs, split hashes, and exact calibration selections accompany the code.
HAR subject roles are fixed; Gas seeds also change the within-batch row
partition. Repeated seeds therefore measure training randomness on HAR and
both split and training randomness on Gas, rather than independent redraws
of people or sensor arrays. The two datasets are engineering classification
benchmarks; their results do not establish deployment or clinical utility.

\subsection{Training, representation, and access-matched controls}

All methods use a one-hidden-layer MLP with 16 tanh units and six softmax
outputs, float64 arithmetic, full participation, batch size 64, SGD learning
rate $0.04$, and the full-parameter penalty $10^{-4}\|w\|^2/2$. There are
20 rounds and five seeds, $11,29,47,71,89$. The primary native-domain matrix
uses $E\in\{1,5,10\}$ local epochs per round; future Gas uses $E=5$.
Each method shares the same per-seed initialization and per-round/client
minibatch permutations. Hyperparameters and reporting round are fixed
before the full experiment; validation pools are unused, and test metrics
are computed only at the final round. Common fixed settings enable a
controlled comparison but do not establish each baseline's best tuned
performance.

Globally scaled methods use a mean and standard deviation estimated from
the pooled source training features. Adapter methods instead use one
diagonal affine standardizer per training client. Scales are floored at
5\% of the pooled source standard deviation, with an absolute numerical
floor of $10^{-8}$. Every adapter is frozen throughout training. A known
participating Gas client uses its own training adapter at inference; an
unseen HAR subject or future Gas batch uses the pooled source-training
adapter. Calibration records use the same native-identity rule. Thus the
adapter methods personalize known domains' representations; they do not
adapt to unseen targets. This source-only fallback can leave an important
train--target representation mismatch. Neither moment standardization nor
good accuracy establishes the latent affine model of
Section~\ref{sec:safeguarded}.

In figures and tables, \covfedcal denotes the safeguarded curve-aggregation
extension (canonical experiment ID \texttt{SafeCurve}); \covfedadapt adds the
fixed client affine adapters (ID \texttt{AdapterSafeCurve}). The descriptive qualifiers distinguish these
reference methods from the main local-curve method. These natural-domain
experiments evaluate those extensions; the public-validation-guard variant
of Section~\ref{sec:algorithm} is evaluated in the controlled pilot in
Section~\ref{sec:experiments}.

The eleven methods are FedAvg, FedProx ($\mu=0.01$), SCAFFOLD, SafeAvg,
SafeCal, \covfedcal, AdapterAvg, AdapterSafeAvg, AdapterSafeCal,
\covfedadapt, and a pooled Centralized reference. SCAFFOLD uses
persistent client/global control variates and Option II\cite{karimireddy2020scaffold}.
We retain sample-weighted objectives and native client sizes, using each
client's actual minibatch count in its control update. This is an explicit
unequal-step extension of the common-step pseudocode; the original paper's
rates are not asserted for this implementation. SCAFFOLD receives no public
calibration labels and its extra control-vector communication is counted.
Centralized SGD sees the same pooled source-training records and performs
$E$ epochs per round equivalent, with the same batch size and learning
rate; it is a pooled-data reference without federated access constraints.

SafeAvg selects between the FedAvg proposal and the certified full-client
gradient fallback from Section~\ref{sec:safeguarded}. \covfedcal adds the
curve-centroid proposal; SafeCal instead adds direct calibration training
of the anchor. Their adapter versions change only the fixed representation.
Backtracking starts at $0.04$ each round and halves the step until the
exact weighted training objective satisfies Equation~\eqref{eq:safe-armijo}.
Every candidate is compared using that same objective. A finite limit of
80 trial steps causes an error on failure, rather than accepting an
uncertified update. Complete per-round certificates are retained.

Every calibrated method receives the same deterministic uniform sample of
128 labeled records from its reserved source calibration pool. Curve fitting
takes one simultaneous control-gradient step with learning rate $0.4$,
midpoint penalty $0.001$, and quadrature points $1/4,1/2,3/4$, followed by
the analytic centroid. Direct calibration takes $3K$ full-calibration SGD
steps at learning rate $0.04$, matching the curve's public gradient calls
and processed calibration examples. This is 45 calls per HAR round, 30
per participating-Gas round, and 24 per future-Gas round. Objectives and
effective step sizes still differ. Calibration-gradient evaluations include
the same full-parameter quadratic penalty as the local models. The two calibrated safeguards share
the same full-client oracle capability, while their evolving trajectories
can have different line-search costs. SafeAvg isolates the safeguard without
extra labels or curve fitting. These controls are essential because the
theorem's descent mechanism does not require a useful curve.

\subsection{IID control, diagnostics, and reported quantities}

At $E=5$, a companion IID-training control shuffles the exact source-training
record pool across clients while preserving every client size. Native
calibration, validation, and test records remain unchanged. This changes
client membership without introducing synthetic feature transformations or
making the target task easier. The control runs FedAvg, \covfedcal,
AdapterAvg, \covfedadapt, and Centralized. For globally scaled methods, the same
pooled source standardizer is retained, so native versus IID compares
training organization under identical held-out targets. Adapter-based IID
clients have new identities and use the global source fallback on native
test domains; their native--IID gap also reflects an inference-policy
change and must not be interpreted as a pure optimization effect. IID
training does not remove the unseen-subject or future-batch target shift.

Training-only diagnostics include weighted client label-prior total
variation from the pooled prior, standardized feature-mean dispersion,
class-conditional feature-mean dispersion with recorded class/domain
support, exact initial gradient dissimilarity, and local-update disagreement.
Class-conditional diagnostics distinguish some feature variation from
unconditional shifts caused by label proportions; they still describe only
moments and do not characterize full joint distributions. Metrics are
reported before and after adapters and for the IID control. Neither
diagnostic values nor baseline outcomes select a new client partition.

Final metrics include sample-weighted test accuracy, macro domain accuracy,
worst-domain accuracy, cross-entropy without the parameter penalty,
class-balanced domain accuracy, and worst-domain class-balanced accuracy.
Balanced accuracy averages recall over classes observed within a domain;
the accompanying class supports, six-by-six confusion matrices, and
per-class recalls make absent classes explicit. Means and sample standard
deviations use the five complete seeds. Paired per-seed curve differences
are reported against FedAvg, the safeguard alone, and matched calibration;
five seeds do not provide a broad domain-population confidence statement.

The cost ledger separates local gradients, public calibration gradients,
full-source gradients, candidate/source-loss queries, final test evaluation,
and centralized work. It counts model, control-vector, gradient, and
candidate payloads and scalar losses as float64 bytes, excluding transport
overhead and network latency. Offline heterogeneity diagnostics are recorded
separately. Exact safeguard access is a stronger communication and statistic
interface than ordinary FedAvg; a gain cannot be presented as a free
aggregation change. No privacy mechanism is implemented. In particular,
diagonal input standardization is not FedBN, and neither FedBN nor FRAug
is claimed as a reproduced numerical baseline here.

\subsection{Results and scope}
\label{subsec:natural-results}

Table~\ref{tab:natural-results} reports the complete $E=5$ comparison;
Figure~\ref{fig:natural-convergence} shows the source objective and local-work
sensitivity. The native training domains have both feature and label differences. Across
the five seeds, weighted mean label-prior total variation is $0.052$ for HAR
and $0.153$ for participating Gas; class-conditional mean dispersion is
$0.415$ and $0.445$, respectively, in the globally standardized
representation. Thus the measured feature differences are not explained
only by different class frequencies. Per-client standardization makes
unconditional client means coincide by construction, but leaves conditional
dispersion of $0.266$ and $0.374$. That zero unconditional mean gap is not
evidence that joint distributions have been aligned.

The IID-training control makes the optimization effect clear for Gas.
At $E=5$, FedAvg's macro balanced accuracy is $88.88\%$ on native batches
and $95.48\%$ after shuffling the same source-training pool, a gain of
$6.60$ percentage points with the same held-out targets. Worst-domain
balanced accuracy rises from $56.59\%$ to $85.08\%$. Class-conditional
feature dispersion falls from $0.445$ to $0.063$, and initial gradient
dissimilarity from $0.740$ to $0.018$. This is a substantial native-grouping
effect; high ordinary accuracy alone would obscure the poorest domain.
HAR has a much smaller native--IID FedAvg gap: macro balanced accuracy
is $91.58\%$ versus $91.87\%$. Its class-conditional dispersion falls
from $0.415$ to $0.098$ and gradient dissimilarity from $0.233$ to $0.028$,
yet both organizations learn useful predictors. Recorded heterogeneity
therefore does not imply a large accuracy penalty on every task.

The learned curves are not competitive with the strongest comparator in
the native Gas task. At $E=5$, globally scaled \covfedcal obtains
$89.01\pm0.55\%$ macro balanced accuracy, compared with FedAvg's
$88.88\pm0.58\%$, SafeCal's $90.78\pm1.43\%$, SCAFFOLD's
$94.75\pm0.27\%$, and the pooled reference's $96.82\pm0.54\%$.
The curve's $0.13$-point gain over FedAvg does not close the native--IID
gap and is smaller than direct calibration's improvement. At $E=10$,
SCAFFOLD reaches $96.14\%$ macro balanced accuracy, versus $89.36\%$
for \covfedcal and $89.03\%$ for FedAvg. Thus the data expose a real
weakness of native FedAvg without establishing that the proposed curves
are an effective solution to it.

On HAR, \covfedcal obtains $91.71\pm0.39\%$ macro balanced accuracy at
$E=5$, FedAvg $91.58\pm0.60\%$, SafeCal $91.78\pm0.59\%$, and
SCAFFOLD $91.83\pm0.72\%$. At $E=10$, \covfedcal is slightly below
FedAvg ($91.61\%$ versus $91.84\%$). These small, mixed differences
provide no consistent advantage attributable to curves or to the
safeguard alone. Paired differences are supplied rather than treating
the five seeds as evidence of universal superiority.

Adapters have a mixed effect on participating domains and a clear limitation
on unseen ones. In native Gas at $E=5$, \covfedadapt improves worst-domain
balanced accuracy from the globally scaled curve's $57.14\%$ to $69.85\%$,
but its macro balanced accuracy is only $89.16\%$; AdapterAvg already
achieves $89.08\%$, and SCAFFOLD's worst-domain score is $80.95\%$.
On future Gas, FedAvg, \covfedcal, and SCAFFOLD obtain $76.51\%$,
$76.59\%$, and $76.36\%$ macro balanced accuracy, whereas
\covfedadapt falls to $66.95\%$. The pooled reference reaches only
$77.08\%$ on these held-out batches. This task includes target shift
that correcting source-client optimization alone does not remove. The
adapter fallback is source only and need not match representations learned
with per-client standardization. These outcomes argue against using
moment alignment as an automatically valid coordinate correction.

Every one of the 4,600 safeguarded rounds satisfies its recorded exact
source-objective comparison: the minimum numerical descent slack is
$1.68\times10^{-4}$ and the largest objective change is
$-3.52\times10^{-4}$. The gradient fallback is selected in 367 rounds;
otherwise another evaluated candidate supplies at least its certified
descent. This verifies the implemented safeguard on the completed runs,
while Theorem~\ref{thm:safeguarded} supplies the exact-arithmetic result.
Neither observation implies that a safeguarded trajectory dominates an
independent FedAvg trajectory or generalizes better. For example, at $E=5$
on HAR, \covfedcal's final source objective is $0.1420$, compared with
FedAvg's $0.1405$, despite the safeguard's monotonicity along its own path.

The additional access is substantial. At $E=5$, \covfedcal processes
115,200 public-calibration gradient examples on HAR and 76,800 on native
Gas, exactly matching SafeCal. Its additional full-source gradient passes
process 101,320 and 166,600 examples, and candidate-loss passes process
303,960 and 499,800 examples, respectively. These counts supplement ordinary
local training rather than replacing it. The algorithm's parameter/statistic
payload is approximately $3.50$ times FedAvg's in both settings, versus
$2.00$ times for SCAFFOLD; baseline loss-monitoring traffic is accounted for
separately. Computation and communication costs therefore reinforce the
conclusion that this implementation does not establish a practical curve
advantage over the tested controls. The complete study contains 435 method
runs and 8,700 rounds; its recorded single-thread runtime is approximately
288 seconds in the supplied environment, excluding download and the
independent reproduction.

\begin{table}[tbp]
\centering
\caption{Recorded native-domain tasks at $E=5$, round 20. Macro BA averages
domain balanced accuracies; worst BA is the poorest domain. Entries are mean
$\pm$ sample standard deviation across five seeds. All eleven methods are
shown, including degraded adapter results on future domains.}
\label{tab:natural-results}
\small
\resizebox{\textwidth}{!}{\begin{tabular}{lrrrrrr}
\toprule
 & \multicolumn{2}{c}{HAR, unseen people} & \multicolumn{2}{c}{Gas, known batches} & \multicolumn{2}{c}{Gas, future batches} \\
Method & Macro BA & Worst BA & Macro BA & Worst BA & Macro BA & Worst BA \\
\midrule
FedAvg & $91.6\pm0.6$ & $80.2\pm2.0$ & $88.9\pm0.6$ & $56.6\pm2.1$ & $76.5\pm1.7$ & $71.6\pm3.4$ \\
FedProx & $91.6\pm0.6$ & $80.3\pm2.0$ & $88.8\pm0.6$ & $56.2\pm1.6$ & $76.5\pm1.7$ & $71.6\pm3.4$ \\
SCAFFOLD & $91.8\pm0.7$ & $80.8\pm3.3$ & $94.7\pm0.3$ & $81.0\pm4.8$ & $76.4\pm1.4$ & $71.1\pm2.9$ \\
SafeAvg & $91.5\pm0.6$ & $80.4\pm1.8$ & $88.8\pm0.5$ & $56.6\pm2.1$ & $76.6\pm1.5$ & $71.8\pm3.1$ \\
SafeCal & $91.8\pm0.6$ & $80.2\pm2.4$ & $90.8\pm1.4$ & $66.1\pm10.5$ & $76.5\pm1.8$ & $71.8\pm3.7$ \\
CovFedAvg (calibration) & $91.7\pm0.4$ & $80.4\pm2.1$ & $89.0\pm0.6$ & $57.1\pm1.8$ & $76.6\pm1.6$ & $71.9\pm3.3$ \\
\midrule
AdapterAvg & $90.8\pm0.6$ & $82.3\pm1.3$ & $89.1\pm1.1$ & $68.8\pm8.2$ & $66.9\pm1.9$ & $64.0\pm1.1$ \\
AdapterSafeAvg & $90.7\pm0.5$ & $82.3\pm1.2$ & $89.1\pm1.0$ & $70.3\pm5.7$ & $66.8\pm1.9$ & $64.0\pm1.2$ \\
AdapterSafeCal & $90.9\pm0.3$ & $81.2\pm1.6$ & $90.1\pm1.7$ & $70.6\pm7.9$ & $68.1\pm3.9$ & $65.6\pm3.3$ \\
CovFedAvg (calibration + adapters) & $91.1\pm0.6$ & $82.2\pm1.1$ & $89.2\pm1.1$ & $69.8\pm5.7$ & $67.0\pm1.9$ & $64.0\pm1.2$ \\
\midrule
Centralized & $92.9\pm0.4$ & $80.2\pm2.7$ & $96.8\pm0.5$ & $88.8\pm3.8$ & $77.1\pm1.8$ & $72.6\pm3.4$ \\
\bottomrule
\end{tabular}
}
\end{table}

\begin{figure}[tbp]
\centering
\includegraphics[width=0.92\textwidth]{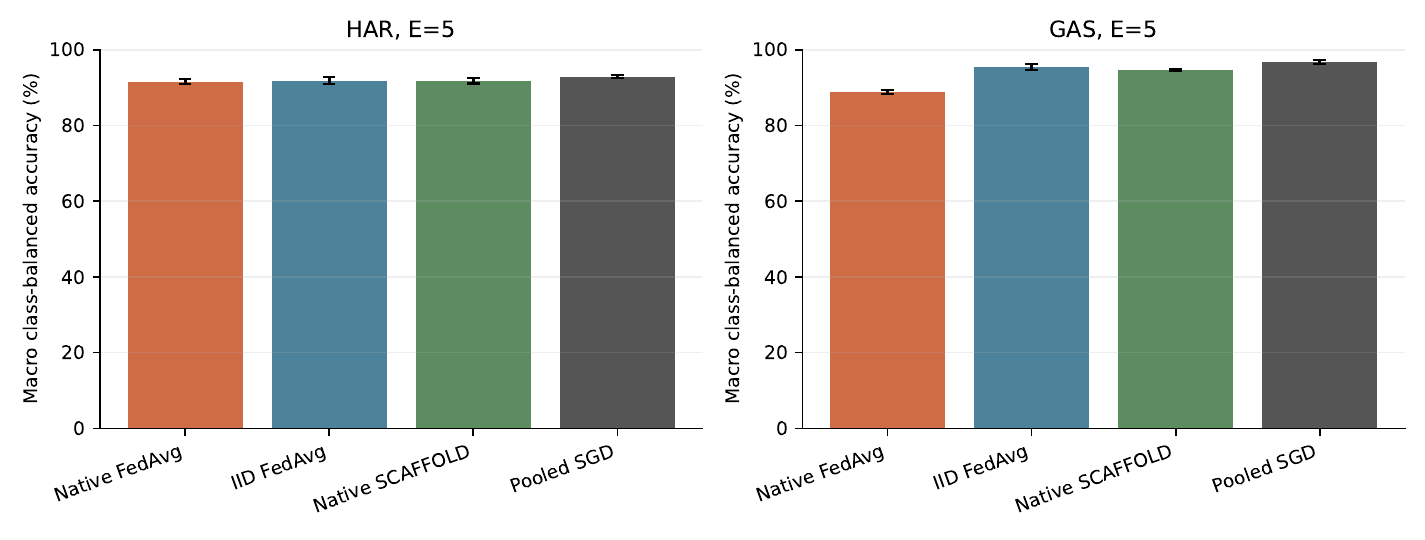}
\caption{At $E=5$, FedAvg under native client grouping and an IID-training
reshuffle of the identical source records and client sizes; native SCAFFOLD
and pooled Centralized are references. Bars show macro balanced accuracy
with sample standard deviations across five seeds. Held-out native targets
and the source-only global standardizer remain fixed. The much larger Gas
gap shows that native heterogeneity can affect optimization; HAR retains
strong FedAvg performance despite measured subject differences.}
\label{fig:natural-iid}
\end{figure}

\begin{figure}[t]
\centering
\includegraphics[width=0.98\textwidth]{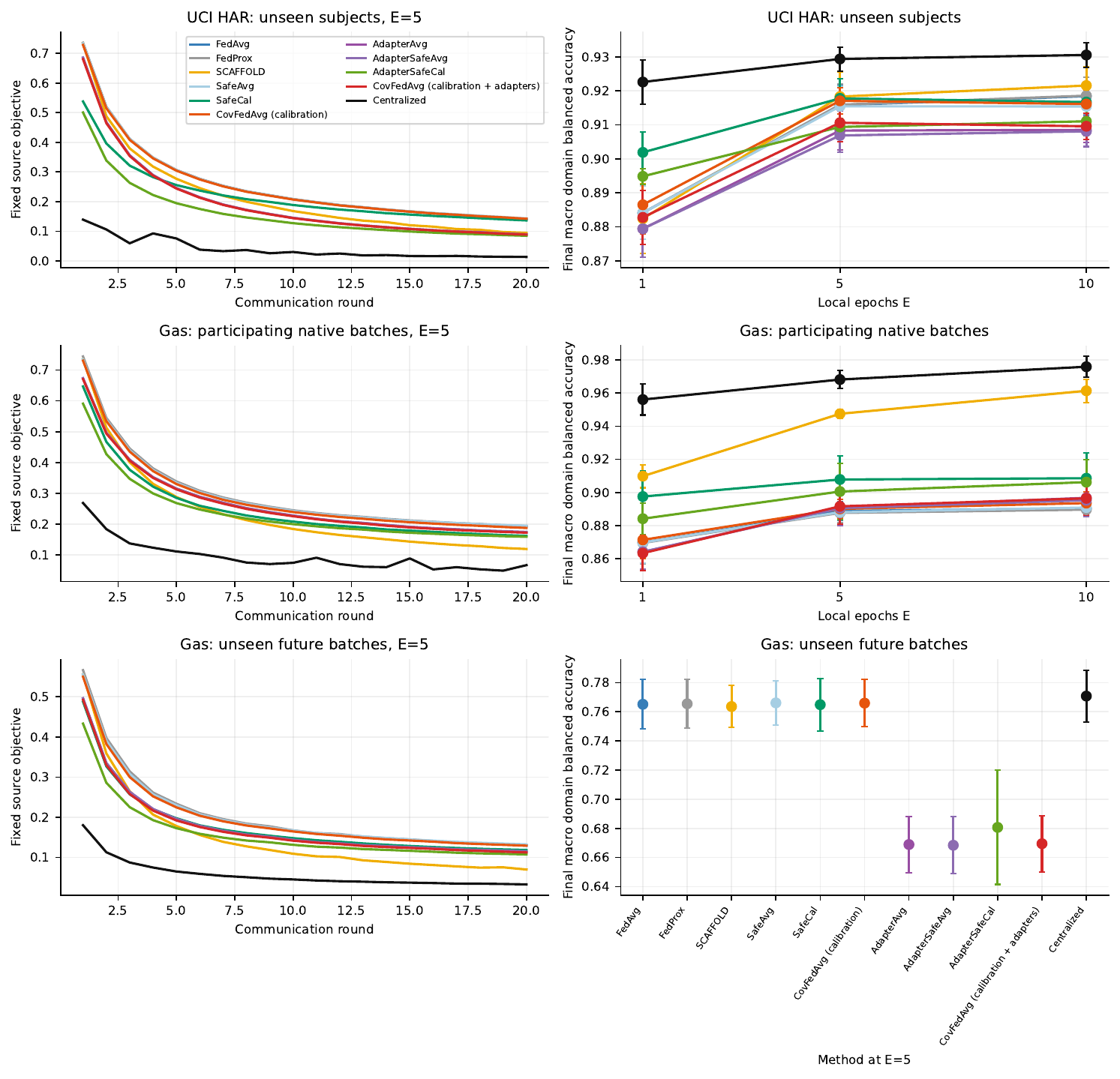}
\caption{Recorded native-domain experiments. \covfedcal is safeguarded
CovFedAvg (\texttt{SafeCurve}); \covfedadapt additionally uses client affine
adapters (\texttt{AdapterSafeCurve}). Left: regularized source-training
objective across rounds at $E=5$. Right: final held-out macro domain balanced accuracy
against local epoch count, with sample standard deviations across five seeds.
The future-Gas task has only $E=5$. The safeguard certifies the source
objective; held-out accuracy is a separate outcome.}
\label{fig:natural-convergence}
\end{figure}

The full machine-readable matrix, per-domain confusion matrices and recalls,
paired differences, adapter parameters, exact source-objective certificates,
and computation ledger are retained in \texttt{results/natural/}. The
reproduction command and source provenance are documented in the accompanying
natural-domain README. These experiments test useful behavior under recorded
heterogeneity; they do not validate an invertible relationship between natural
domains or guarantee that curves improve generalization.

\FloatBarrier
\section{Interpretation, limits, and next tests}\label{sec:discussion}
CovFedAvg changes both the geometric operation and the information model:
client-local curves guide a shared endpoint, and distributed training losses
select an actual model without server calibration examples. The absolute-value
diagnostic illustrates the intended complementary-support mechanism. The
larger experiments show modest improvements over FedAvg, particularly on
participating Gas batches, but the matched direct distributed optimizer
remains stronger there. A training-loss selector also permits worse held-out
or worst-domain performance. Curve fitting is therefore a tested proposal,
not an established source of superior generalization.

The paragraphs below discuss the retained public-calibration reference study.
The native-domain evaluation changes the empirical conclusion from a
controlled feasibility claim to a measured assessment of heterogeneity.
Gas batch membership substantially changes FedAvg's performance under an
otherwise fixed held-out task, whereas HAR permits strong FedAvg results
even on unseen subjects. Non-IID sampling is a property of the data,
not a rule that FedAvg must fail. Both results are retained. The
supplementary imposed-shift pilot remains useful for checking the code,
but does not substitute for these natural domains.

The safeguard establishes sufficient decrease of its actual finite-source
objective and has a stationarity proof. Its full-client gradient fallback
supplies that proof even if curve proposals are unhelpful. The natural
results show why a convergence theorem cannot establish the empirical
contribution of curve fitting: calibration-only safeguards can perform
better, and SCAFFOLD closes much more of the participating-Gas optimization
gap. We do not attribute the generic fallback's guarantee to a discovered
covariant connection. Parameter symmetries also make weight geometry
representation dependent. The centroid of a low-loss path can still have
high loss, which is why the bound retains a dispersion term and the
implementation evaluates the actual candidate.

The calibration-only and no-curve controls address whether curves offer a
benefit beyond extra data and safeguard queries. Their mixed or negative
results do not support a general curve-specific advantage. FedBN and FRAug
remain relevant additional comparators, with different architectures and
data-access requirements. We have not reproduced those methods here, and
no comparative superiority over them is claimed. Calibration scarcity and
domain coverage, hyperparameter selection using source-only validation,
larger models, and incomplete client participation remain important tests.
The new study retains native client-size and class imbalance; five seeds
still do not constitute replication across independent sensor arrays or
populations of acquisition sites.

\begin{figure}[t]
\centering
\includegraphics[width=.8\linewidth]{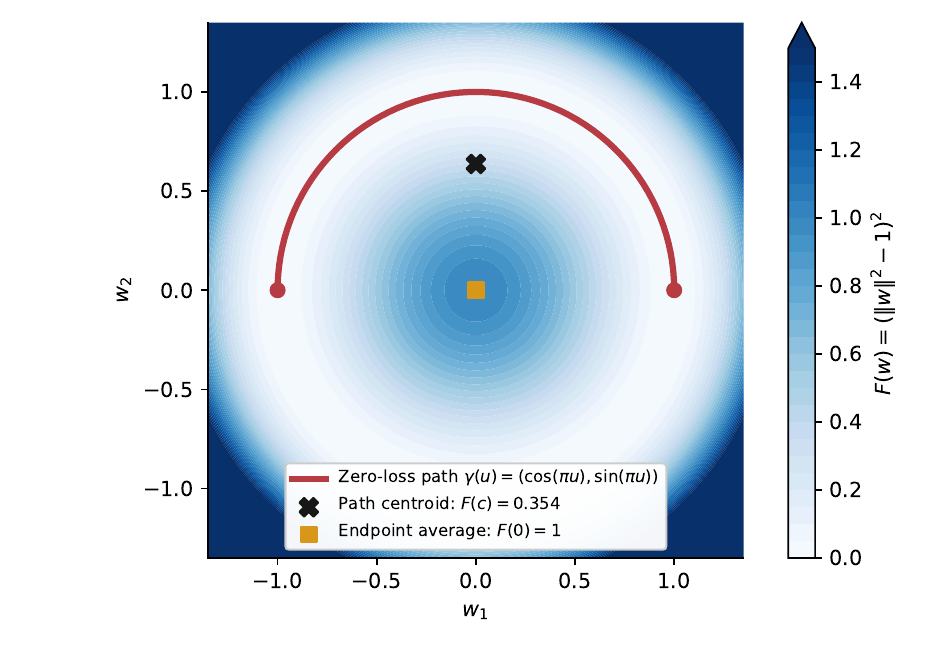}
\caption{An explicit limit of curve aggregation. Every point of the upper
unit semicircle has zero loss for $F(w)=(\|w\|^2-1)^2$, but the path centroid
$(0,2/\pi)$ has loss $(1-4/\pi^2)^2\approx0.354$. An improvement over the
endpoint average does not make the centroid a zero-loss solution.}
\label{fig:counterexample}
\end{figure}

Fixed per-client moment adapters expose a second limitation. They can
remove class-dependent information when native class proportions differ,
and no target-domain adapter can be inferred without target information.
Their deterioration on future Gas illustrates this risk. The exact affine
identity applies to a specified latent measurement model; matching
unconditional moments does not establish that model. A more targeted next
study would estimate class-conditional alignment using source data and
measure its residual gradient disagreement, with a separate held-out
population for final evaluation. Adding a known drift-correction optimizer
to the local updates is also a useful ablation, provided its contribution
is distinguished from that of curves.

The supplementary image and tabular examples model intensity calibration
and measurement scales. The completed HAR~\citep{reyesortiz2013har} study
uses actual subject differences, and Gas~\citep{vergara2012gas} uses actual
acquisition periods. Hospital domains in WILDS~\citep{koh2021wilds} and
cross-style object recognition in DomainNet~\citep{peng2019domainnet} would
extend evaluation to natural shifts involving different institutions or
visual domains, without assuming invertible inter-domain transformations.
A text-domain study needs its own discrete or latent transformation model.
The parameter operations permit other modalities with compatible
architectures, but NLP effectiveness has not been tested.

For both constructions, keeping local training examples on clients is a data-flow property, not a
formal privacy guarantee. Client weights and additional training-risk,
gradient, and preprocessing summaries can leak information. Public
calibration examples also require an independently justified source.
A private deployment would require a threat model and mechanisms compatible
with the required queries. We make no privacy or clinical-utility claim.

\section{Conclusion}
Local feature coverage can leave multiple parameter configurations compatible
with one client's observations. CovFedAvg explores that freedom through
local paths to a shared candidate and selects a model using distributed
losses, eliminating the centralized calibration requirement. Exact candidate
selection gives an empirical comparison, while preservation of every client
remains conditional and stationarity requires additional descent conditions.
The motivating regression is solved; naturally separated domains show
modest gains over FedAvg but no reliable advantage over matched direct
optimization. Correct transformation identities, centroid counterexamples,
and a separate safeguarded stationarity theorem delimit the claims. Further
work must establish when the path constraints improve optimization or
generalization and provide an end-to-end protected communication protocol.

\FloatBarrier

\appendix
\section{Controlled-shift feasibility pilot}\label{sec:experiments}
This initial pilot is retained as supplementary evidence. Its client shifts
are imposed transformations of public records; the natural-domain evaluation
in the main text uses actual subject and acquisition-period identities.
We test whether the corrected construction can be executed and whether any
benefit exceeds using its public calibration data directly. This is a small
federated simulation, with synthetic measurement domains imposed on two real
benchmark datasets and on one synthetic process. It is not evidence of
invariance, transformation identification, a uniformly low-loss path, or
improvement in every client's risk.

\subsection{Tasks, domains, and held-out data}
We use the 1,797-record, ten-class, $8\times8$ digits dataset bundled with
scikit-learn\cite{alpaydin1998optdigits,sklearn2026digits}, and the 569-record,
30-feature, binary Wisconsin Diagnostic Breast Cancer
dataset\cite{wolberg1993wdbc,sklearn2026wdbc}. The simulated sites do not
correspond to actual hospitals or independently collected handwriting domains.
A third example generates 2,400 records from three standard-normal latent
variables linearly mixed into 12 measurement channels, with Gaussian channel
noise of standard deviation $0.12$. The first four channels additionally
contain $0.3z_0z_1$. A three-state label is the argmax of
$(z_0+0.5z_1,-z_0+0.5z_2,-z_1-0.5z_2)$ after adding independent Gaussian
score noise of standard deviation $0.2$. Generation uses seed 2026; subsequent
splits and model initializations vary across seeds.

Each experiment has four clients. Digits are centered as $x/16-0.5$.
The IID control uses this representation unchanged. In the shifted-digit
experiment, client gains are $(0.5,0.8,1.2,1.7)$ and offsets are
$(-0.6,-0.2,0.3,0.6)$, applied uniformly to pixels without clipping.
For cancer and sensors, canonical features are standardized using only the
public calibration subset before simulating client transformations. With
zero-based feature index $j$ and client index $k$, these transformations have
gain $\exp[0.60\cos(0.7j+k\pi/2)]$ and offset
$0.65\sin(0.5j+k\pi/2)$. They are invertible and label preserving. The common
public scaling convention and representative labeled public samples are
additional assumptions; the server does not estimate or receive the inverse
client transformations.

For each seed, we first hold out a stratified 25\% test set, then reserve
disjoint stratified public calibration and public validation subsets. Counts
are respectively 1,187 private training, 80 calibration, 80 validation, and
450 test records for digits; 298, 64, 64, and 143 for cancer; and 1,608, 96,
96, and 600 for sensors. Each split is separately partitioned approximately
equally by class among the four domains. Public calibration and validation
have 20, 16, or 24 examples per domain for digits, cancer, or sensors.
Each record occurs in one split and one client domain, with the same site
transformation applied to that client's training and test inputs. Aggregation
weights use private training counts. The server accesses only transformed
labeled public examples and uploaded model parameters. Test metrics never
choose hyperparameters, updates, checkpoints, or the reported round.

\subsection{Matched training and calibration budgets}
All methods use a dense MLP with 32 tanh hidden units and softmax output,
float64 arithmetic, full participation, and the same per-seed model
initialization. Each client performs two local epochs of SGD with batch size
32 and learning rate $0.06$, starting from its method's broadcast model.
Minibatch permutations are determined by round, client, and seed and are
shared across methods. There are 20 rounds and three seeds: 11, 29, and 47.
There is no momentum or weight decay. We fix hyperparameters before the
reported run; this pilot does not optimize each baseline's hyperparameters.

We compare FedAvg; FedProx with proximal coefficient $0.01$; FedAvg followed
by direct public calibration (FedAvg+cal); and the corrected CovFedAvg.
For CovFedAvg, controls start at $(a+w_k)/2$ and minimize
\begin{equation}
 J(\Theta)=\sum_k\alpha_k\left[
 \frac{1}{3}\sum_{u\in\{1/4,1/2,3/4\}}
 \widehat F_{\rm cal}(\gamma_k(u))
 +\frac{0.001}{2}\left\|\theta_k-\frac{a+w_k}{2}\right\|^2\right].
\end{equation}
We use three simultaneous gradient steps of size $0.40$, retaining
$\alpha_k$ in each control gradient and fixing both endpoints. Thus each
round uses $4\times3\times3=36$ full-calibration forward/backward evaluations.
FedAvg+cal instead takes 36 full-calibration gradient steps at learning rate
$0.06$, starting at the same type of FedAvg anchor. This matches server
gradient calls and processed calibration examples; effective update sizes and
objectives still differ. Both calibrated methods accept the final candidate
only if its finite pooled public-validation cross entropy is no greater than
the current anchor's. The distinct validation subset is reused across rounds,
so its empirical guard is not an untouched population-risk certificate.

The four methods transmit the same parameter-vector payloads. Parameter
counts are 2,410, 1,058, and 515 for digits, cancer, and sensors; aggregate
upload and download payloads are each 77,120, 33,856, and 16,480 bytes per
round in float64, excluding protocol overhead. Over 20 rounds, the calibrated
methods each add 720 server-gradient evaluations, processing 57,600, 46,080,
and 69,120 public examples, respectively. Private local-gradient example
counts are 47,480, 11,920, and 64,320. Public calibration is therefore
substantial computation, especially for the small cancer task. Curve
controls remain on the server. Network latency, privacy mechanisms, and
distributed systems costs are not simulated.

\subsection{Results and their limits}
Table~\ref{tab:pilot-results} reports final-round means and sample standard
deviations across seeds. Macro accuracy averages client accuracies equally;
worst accuracy is the minimum client accuracy within each run; test cross
entropy pools held-out observations. The deviations are not confidence
intervals or significance tests. Figure~\ref{fig:pilot-convergence} shows
convergence using the same test sets only for reporting.

\begin{table}[t]
\centering
\caption{Feasibility pilot at round 20: mean $\pm$ sample standard deviation
over three seeds. Domains in the shifted settings are constructed synthetic
measurement domains.}\label{tab:pilot-results}
\small
\begin{tabular}{llrrr}
\toprule
Setting & Method & Macro acc. (\%) & Worst acc. (\%) & Test CE \\
\midrule
Digits, IID & FedAvg & $94.1\pm1.5$ & $91.1\pm4.0$ & $0.330\pm0.032$ \\
Digits, IID & FedProx & $94.1\pm1.5$ & $91.1\pm4.0$ & $0.331\pm0.032$ \\
Digits, IID & FedAvg+cal & $94.2\pm0.7$ & $91.7\pm1.4$ & $0.228\pm0.018$ \\
Digits, IID & CovFedAvg & $94.6\pm0.7$ & $91.7\pm2.2$ & $0.311\pm0.029$ \\
Digits, shifted & FedAvg & $93.2\pm0.9$ & $90.0\pm1.8$ & $0.422\pm0.039$ \\
Digits, shifted & FedProx & $93.2\pm0.9$ & $90.0\pm1.8$ & $0.424\pm0.039$ \\
Digits, shifted & FedAvg+cal & $94.1\pm0.3$ & $91.7\pm1.0$ & $0.256\pm0.028$ \\
Digits, shifted & CovFedAvg & $93.9\pm1.0$ & $91.7\pm2.7$ & $0.395\pm0.035$ \\
Breast cancer, shifted & FedAvg & $95.3\pm2.8$ & $92.6\pm3.2$ & $0.136\pm0.041$ \\
Breast cancer, shifted & FedProx & $95.3\pm2.8$ & $92.6\pm3.2$ & $0.136\pm0.041$ \\
Breast cancer, shifted & FedAvg+cal & $96.0\pm1.8$ & $91.7\pm2.8$ & $0.141\pm0.057$ \\
Breast cancer, shifted & CovFedAvg & $95.8\pm2.1$ & $91.7\pm2.8$ & $0.130\pm0.036$ \\
Synthetic sensors, shifted & FedAvg & $91.2\pm1.1$ & $88.2\pm1.0$ & $0.215\pm0.009$ \\
Synthetic sensors, shifted & FedProx & $91.2\pm1.1$ & $88.2\pm1.0$ & $0.215\pm0.009$ \\
Synthetic sensors, shifted & FedAvg+cal & $91.4\pm1.1$ & $88.9\pm0.8$ & $0.213\pm0.010$ \\
Synthetic sensors, shifted & CovFedAvg & $91.4\pm1.0$ & $88.4\pm1.0$ & $0.214\pm0.009$ \\
\bottomrule
\end{tabular}

\end{table}

\begin{figure}[p]
\centering
\includegraphics[width=0.95\textwidth]{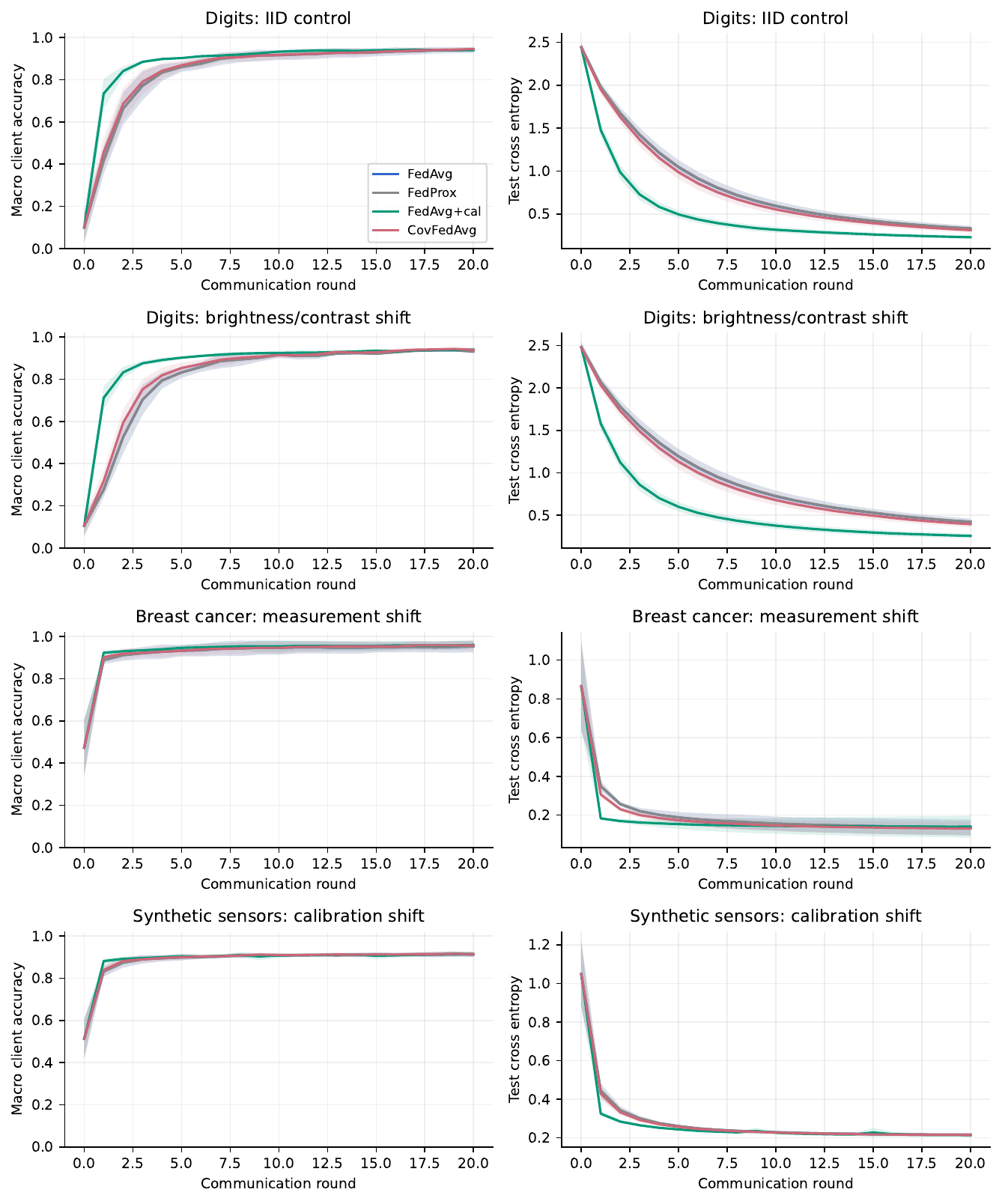}
\caption{Held-out macro client accuracy and cross entropy. Lines are
three-seed means; bands show one sample standard deviation. Public
calibration updates have the same gradient-evaluation budget for FedAvg+cal
and CovFedAvg.}\label{fig:pilot-convergence}
\end{figure}

CovFedAvg's final macro accuracy is slightly higher than FedAvg's in all four
settings under this fixed configuration. However, direct calibration equals
or exceeds its macro accuracy in all three shifted settings, and achieves
lower digit cross entropy: $0.256$ versus $0.395$ under shift and $0.228$
versus $0.311$ in the IID control. Consequently, these results do not
establish a reliable advantage attributable to curve aggregation rather than
public calibration. FedProx at the tested coefficient produces the same final
predicted classes as FedAvg, although its losses differ. For cancer,
CovFedAvg improves macro accuracy from $95.34\%$ to $95.81\%$ while mean
worst-client accuracy decreases from $92.59\%$ to $91.67\%$. An empirical
pooled validation guard does not guarantee per-client or held-out
improvement.

The executable implementation, fixed configurations, record indices,
per-round metrics, per-client final metrics, paired seed differences, costs,
and an unguarded-candidate diagnostic accompany the manuscript. The latter
evaluates each candidate on the \emph{guarded trajectory}; it is not an
independent experiment that removes the guard from all previous rounds.
Eight tests pass, covering endpoint continuity, the integral centroid,
unequal-weight FedAvg recovery, finite-difference network and control
gradients, guard behavior including nonfinite rejection, disjoint splits and
invertible transforms, and deterministic repeated rounds. Nine separate
numerical witnesses complement the written mathematical audit, giving 17
passing checks in total; they do not replace proofs or verify the theoretical
assumptions on these datasets.

The bundled experiment README records exact commands and dependencies.
The full single-thread NumPy benchmark took approximately $3.48$ seconds
including plotting and outputs in the recorded environment. These small
tasks omit real multi-institution shifts, larger models, stronger
personalization baselines, calibration scarcity and mismatch, independently
unguarded training, partial participation, and privacy analysis. Genuine
FedBN, which retains client-local batch-normalization parameters and
statistics, is not implemented; ordinary input scaling is not presented as
that method. Three seeds, one configuration, synthetic label-preserving
shifts, and the assumed public reference data limit generalization. These
examples establish executability and expose failures; they do not establish
universal superiority or clinical utility.

\FloatBarrier

\bibliographystyle{plainnat}
\bibliography{references}
\end{document}